\documentclass[11pt]{article}                 % 11pt 字号，article 文档类
\usepackage{geometry}                         % 控制页面边距
\usepackage{graphicx}                         % 插入图片
\usepackage{subcaption}                       % 子图环境 (subfigure, subtable)
\usepackage{xcolor}                           % \color{} \textcolor{颜色}{文字>{text}
\usepackage{mathtools}                        % 增强版 amsmath（自动加载 amsmath）
\usepackage{amssymb}                          % Extra math symbols (\mathbb, \mathcal)
\usepackage{amsthm}                           % Theorem, lemma, proof environments
\usepackage{bbm}                              % \mathbbm{1} 示信函数
\usepackage{dsfont}                           % \mathds{1} 左边加粗的 1
\usepackage{enumitem}                         % 自定义 enumerate/itemize 序号
\usepackage{natbib}                           % 作者-年份 \citet, \citep, \citealt
\usepackage[colorlinks,                       % 开启彩色链接
linkcolor=blue,                               % 目录、公式、图表交叉引用
citecolor=blue,                               % 文献引用
urlcolor=blue                                 % URL
]{hyperref}
\usepackage{cleveref}                         % 自动加Theorem等前缀，需在hyperref之后！
\crefname{equation}{}{}                       % \cref{} 公式加括号
\usepackage{algorithm}                        % 算法浮动体 (带标题和编号)
\usepackage{algorithmic}                      % 算法伪代码环境 (较老的语法)
\newtheorem{theorem}{Theorem}%[section]
\newtheorem{lemma}{Lemma}[section]
\newtheorem{proposition}{Proposition}%[section]
\newtheorem{corollary}{Corollary}[section]
\newtheorem{definition}{Definition}[section]
\newtheorem{remark}{Remark}[section]

\newtheorem{assumption}{Assumption}
\newcommand{\Var}{\mathrm{Var}}
\renewcommand{\tilde}{\widetilde}
\renewcommand{\hat}{\widehat}
\newcommand{\Xs}{\mathcal{X}}

\newcommand{\As}{\mathcal{A}}
\newcommand{\Ds}{\mathcal{D}}
\newcommand{\Ss}{\mathcal{S}}

\newcommand{\Is}{\mathcal{I}}
\newcommand{\Vs}{\mathcal{V}}
\newcommand{\Eb}{\mathbb{E}}
\newcommand{\Pb}{\mathbb{P}}
\newcommand{\Rb}{\mathbb{R}}
\newcommand{\drm}{\mathrm{d}}
\newcommand{\argmax}{\mathrm{argmax}}

\newcommand{\Bern}{\mathrm{Bern}}
\newcommand{\KL}{\mathrm{KL}}

\newcommand{\1}{\mathds{1}}
\newcommand{\Fis}{\widehat{F}_\mathrm{IS}}
\newcommand{\Fisc}{\widehat{F}_\mathrm{ISc}}
\newcommand{\Fwis}{\widehat{F}_\mathrm{WIS}}
\newcommand{\Fdr}{\widehat{F}_\mathrm{DR}}
\newcommand{\Fdrc}{\widehat{F}_\mathrm{DRc}}
\title{Pessimistic Risk-Aware Policy Learning in Contextual Bandits} %universal
\author{Yilong Wan, Yuqiang Li, Xianyi Wu}
\date{\today}

\begin{document}

\maketitle

\begin{abstract}

We study risk-aware offline policy learning, aiming to learn a decision rule from logged data that is optimal under general risk criteria. This problem is crucial in high-stakes domains where online interaction is infeasible and adverse outcomes must be carefully controlled. However, existing literature on offline contextual bandits either centers on expected-reward criteria or restricts risk considerations to policy evaluation instead of optimization. In this work, we propose a unified distributional framework for optimizing Lipschitz-continuous risk functionals, a broad class of risk measures encompassing mean-variance, entropic risk, and conditional value-at-risk, among others. 
By developing novel empirical concentration inequalities for importance sampling-based distributional estimators, our analysis derives data-dependent suboptimality bounds with an $\tilde{\mathcal{O}}(1/\sqrt{n})$ rate, without relying on restrictive uniform overlap assumptions. This rate is minimax optimal and matches that of risk-neutral offline policy optimization, indicating that optimizing general Lipschitz risk criteria incurs no additional statistical cost relative to the expected-reward.

\end{abstract}

\section{Introduction}\label{sec-1}
%%%%%%%%%%%%%%%%%%%%%%%%%%%%%%%%%%%%%%%%%%%%%%%%%%%%%%%%%%%%%%%%%%%%%%%%%%%%%%%%%%%%%%%%
%                                  Sec-1
%%%%%%%%%%%%%%%%%%%%%%%%%%%%%%%%%%%%%%%%%%%%%%%%%%%%%%%%%%%%%%%%%%%%%%%%%%%%%%%%%%%%%%%%

Data-driven individualized decision-making is critical across a wide range of domains. When online experimentation is costly, unethical, or unsafe, decision policies must be learned solely from logged data---a setting termed offline policy learning, whose objective is to drive a policy that achieves near-optimal performance on a target population using data collected by a behavior policy. Such problems are ubiquitous in high-stakes applications, including healthcare, finance, and online platforms \cite{murphy2003optimal, bottou2013acounterfactual, bertsimas2020predictive}. 
While expected reward is the most commonly adopted performance criterion in policy learning, it is frequently insufficient and even misleading in these risk-aware settings. For instance, in medical decision-making, a clinician may prefer a treatment with slightly lower average efficacy but substantially lower outcome variability to reduce the risk of severe adverse events; in regulated financial environments, decision rules must control downside risk and guard against catastrophic losses, often at the expense of average returns. These considerations motivate learning with respect to various risk criteria, such as mean-variance trade-offs \citep{mannor2011mean}, conditional value-at-risk (CVaR) \citep{rockafellar2000optimization}, and entropic risk measures \citep{howard1972risk-sensitive}.

Despite the prevalence of these risk criteria, conventional approaches commonly handle each risk measure in isolation. From a distributional perspective, however, several works in off-policy evaluation have shown that various risk measures can be simultaneously assessed within a unified framework \cite{chandak2021universal,huang2021off,huang2022offpolicy}. By first estimating the cumulative distribution function (CDF) of outcomes under the target policy, risk functionals of interest can be approximated through a plug-in approach. This two-stage paradigm offers a key advantage: statistical guarantees on the CDF estimator automatically extend to all downstream risk functionals, without requiring separate analysis or correction.

In this work, we move beyond evaluation and aim to develop a unified distributional framework for offline policy learning that accommodates a broad class of risk measures. We focus on the contextual bandit setting and consider an offline dataset $\Ds=\{(X_i,A_i,Y_i)\}_{i=1}^n$, where each context--action--reward triple is independently generated by a behavior policy $\beta$, with $\beta(x,a)\coloneq\Pb(A_i=a\mid X_i=x)$ denoting the propensity score. Given a policy class $\Pi$, where each $\pi\in\Pi$ is a deterministic mapping from contexts to actions, our goal is to learn a policy $\pi\in\Pi$ whose performance---measured by a user-specified risk functional $\rho$ applied to the reward distribution $F^\pi$ induced by $\pi$---is as close as possible to that of the optimal policy $\pi^*=\argmax_{\pi\in\Pi}\;\rho(F^\pi)$.

A central challenge in achieving this goal is that existing off-policy CDF estimators typically rely on the importance sampling (IS) method, whose statistical validity requires the target policy $\pi$ to satisfy a overlap condition: $\inf_x\beta(x, \pi(x))>0$. While this assumption is mild when evaluating a fixed policy, it becomes restrictive in policy learning. Because the learner must compare all candidate policies in $\Pi$, a uniform overlap condition, $\inf_{\pi \in \Pi} \inf_{x} \beta(x, \pi(x)) > 0$, is enforced, representing a foundational limitation in offline policy learning. Recent work \cite{jin2025policy} shows that for expected reward maximization, this limitation can be circumvented via pessimistic learning that optimizes lower confidence bounds (LCBs) on policy value, requiring overlap only for the optimal policy, i.e., $\inf_x \beta(x, \pi^*(x)) > 0$. Extending this pessimistic principle to risk-aware objectives, however, introduces non-trivial challenges. First, it demands the construction of valid CDF estimators for policies that may violate the overlap condition, along with rigorous error guarantees that enable high-confidence bounds for general risk functionals. Second, such guarantees must be extended to hold uniformly over the entire policy class $\Pi$, as required for computing LCBs in pessimistic learning paradigms.

We summarize our contributions as follows. 
First, we develop a unified pessimistic framework for offline policy learning under a broad class of risk functionals that are Lipschitz continuous with respect to the sup norm over CDFs. This class encompasses many widely adopted risk criteria, such as mean-variance, CVaR, and cumulative prospect weighting, under the bounded rewards condition. Second, we propose a suite of off-policy CDF estimators that accommodate the absence of overlap, including IS, weighted IS and doubly robust variants, and establish the first empirical DKW-type concentration inequalities that tolerate unbounded importance weights. Under the overlap condition, our data-dependent concentration bounds achieves an $\mathcal{O}(1/\sqrt{n})$ rate, recovering existing results as a special case. Third, we establish uniform high-confidence bounds over the policy class $\Pi$, ensuring the calculation of LCBs. Critically, these LCBs remain finite and informative even for policies lacking overlap, enabling our algorithm to function effectively even when no policy has overlap, where prior pessimistic approach fails. Finally, assuming only the optimal-policy overlap condition, we establish an $\tilde{\mathcal{O}}(1/\sqrt{n})$ suboptimality bound for the learned policy, which is proven to be minimax optimal. This rate matches that of pessimistic learning for expected reward maximization in its dependence on sample size, policy class complexity, and the overlap constant---showing that, for Lipschitz risk functionals, incorporating general risk criteria entails no additional statistical cost compared to mean optimization.

\subsection{Related Works}

\paragraph{Offline policy learning.}

Existing approaches to offline policy learning predominantly adopt a two-stage paradigm: first estimating the performance of candidate policies from logged data, then selecting the policy that maximizes the estimated value over a predefined class \citep{swaminathan2015batch, kitagawa2018who, kallus2018balanced, athey2021policy, zhou2023offline, zhan2024policy}. The performance of such greedy methods is determined by the worst-case overlap, as they require accurate evaluation of all policies. These works therefore all assume a uniform overlap condition, i.e., the behavior policy assigns non-negligible probability to all actions taken by any policy in the class, an assumption commonly violated for rich policy classes in practice. To mitigate this, \citet{swaminathan2015batch} clips tiny propensities by a fixed constant, inevitably introducing undesired bias. \citet{zhao2024positivity} restricts attention to incremental propensity score policies \citep{kennedy2019nonparametric} that are close to the behavior policy, and establishes asymptotic guarantees for this constrained policy optimization.

\citet{jin2025policy} show that the uniform overlap is not necessary for efficient policy learning when the behavior policy is known. Building on the principle of pessimism \citep{buckman2021the,jin2021is}, they propose a pessimistic policy learning algorithm that optimizes LCBs, instead of point estimates, of policy values. The resulting guarantees depend only on the overlap for the optimal policy, thereby improving upon the state-of-the-art statistical rates, and are supported by a generalized empirical Bernstein inequality that accommodates unbounded and dependent data. Our work builds on this pessimistic framework and extends it to general risk measures  beyond expectation. 

\paragraph{Risk-aware reinforcement learning.}

A substantial body of literature in bandit and reinforcement learning (RL) has explored risk-sensitive objectives beyond expected return. Common risk functionals include variance \citep{sani2012risk, tamar2016learning}, mean-variance \citep{mannor2011mean}, CVaR \citep{bastani2022regret, wang2023near}, entropic risk measures \cite{fei2020risk, fei2022cascaded}, and risk families such as cumulative prospect theory risk functionals \citep{prashanth2016cumulative}. Closely related is distributional RL \citep{bellemare2017a, dabney2018distributional}, which represents the entire return of returns and thus is inherently suited for risk-aware RL \citep{dabney2018implicit, keramati2020being, liang2024bridging}.

Despite the extensive online literature above, risk-aware offline RL remains far less studied, largely due to the inherent distribution shift. \citet{nuria2021riskaverse} and \citet{ma2021conservative} adapt distributional RL to the offline setting and mitigate distribution shift by constraining policies or penalizing out-of-distribution actions. \citet{rigter2023one} propose a model-based approach, with theoretical guarantees requiring normally distributed transitions. More recently, \citet{zhang2024pessimism} develop the first provably efficient algorithm for risk-aware offline RL; however, their framework is designed solely for the entropic risk measure, whereas our method accommodates a broader class of risk functionals.

%=======================================================================================
%%%%%%%%%%%%%%%%%%%%%%%%%%%%%%%%%%%%%%%%%%%%%%%%%%%%%%%%%%%%%%%%%%%%%%%%%%%%%%%%%%%%%%%%

\section{Preliminaries}\label{sec-2}
%%%%%%%%%%%%%%%%%%%%%%%%%%%%%%%%%%%%%%%%%%%%%%%%%%%%%%%%%%%%%%%%%%%%%%%%%%%%%%%%%%%%%%%%
%                                  Sec-2
%%%%%%%%%%%%%%%%%%%%%%%%%%%%%%%%%%%%%%%%%%%%%%%%%%%%%%%%%%%%%%%%%%%%%%%%%%%%%%%%%%%%%%%%

This section introduces the problem setup, along with regularity conditions on the risk functionals and the policy class.
\paragraph{Contextual bandit.}
Let $X, A$, and $Y$ denote the context, action and reward in a stochastic contextual bandit, taking values in spaces $\Xs, \As$, and $\Rb$, respectively. We assume a finite action space with $|\As|=K$. A contextual bandit environment $v$ is specified by a context distribution $\Pb_X$ together with a collection of conditional reward distributions $\{\Pb_{Y\mid X=x,A=a}: (x,a)\in\Xs\times\As\}$. In the batched setting, we are given a data set $\Ds = \{(X_i, A_i, Y_i) \}_{i=1}^{n}$ consisting of $n$ i.i.d. samples generated from $v$ using a behavior policy $\beta$. Specifically, at each round, after observing a context $X_i\sim\Pb_X$, the experimenter takes an action $A_i\in\As$ with probability $\beta(X_i, A_i)$, and subsequently receives a reward $Y_i\sim\Pb_{Y\mid X=X_i,A=A_i}$ from the environment. Throughout the paper, we assume that the behavior policy $\beta$ is known, which holds when the data are collected directly by the learner, commonly the case in online learning systems \citep{li2010a}.
\begin{assumption}
The propensity scores $\beta(x,a)$ are known for all $(x,a)\in\Xs\times\As$.
\end{assumption}

\paragraph{Policy learning under risk functionals.}
Let $\Pi$ be a class of deterministic policies, where each $\pi \in \Pi$ maps contexts to actions, $\pi: \Xs \to \As$. The goal of policy learning is to identify an optimal policy within $\Pi$ that maximizes its performance under a given criterion.
We begin by introduce some notation for defining optimality under general risk functionals. For $t\in\Rb$, let
$$
G(t\mid x,a) \coloneq \Pb( Y\leq t \mid X=x,A=a )
$$
denote the conditional CDF of the reward given $(X,A)=(x,a)$, which is independent of policies. For any $\pi\in\Pi$, the marginal CDF of the reward induced by $\pi$ can be expressed as
$$
F^\pi(t)\coloneq \Eb_{X\sim\Pb_X}\,G(t\mid X,\pi),
$$ 
where $G(t\mid x,\pi)$ is shorthand for $G(t\mid x,\pi(x))$. Given a risk functional $\rho$ that maps CDFs to real numbers, we define the performance of $\pi$ under $\rho$ as 
$$
\rho_\pi \coloneq \rho( F^\pi ).
$$
The optimal policy under the risk $\rho$ is then defined as
$$
\pi^* \in \underset{\pi\in\Pi}{\argmax}\; \rho_\pi,
$$
with the convention that larger values of $\rho$ are preferred. For brevity, its dependence on $v$, $\rho$ and $\Pi$ is omitted.

\paragraph{Lipschitz risk functionals.}
We consider a broad class of risk functionals that are Lipschitz continuous with respect to the sup norm over the space of reward distributions \citep{cassel2018a, huang2021off, huang2022offpolicy}. This regularity condition is central to our analysis, as it ensures that uniform errors in estimating reward distributions translate directly into controlled errors in the evaluated risk. 
\begin{definition}\label{def-lipschitz-risk}
A risk functional $\rho$ is said to be $L$-Lipschitz continuous if there exists a constant $L>0$ such that, for any two CDFs $F_1$ and $F_2$, it holds
$$
|\rho(F_1) - \rho(F_2)| \leq L \| F_1 - F_2 \|_\infty,
$$
where $\| f \|_\infty \coloneq \sup_{t\in\Rb} |f(t)|$ is the sup norm.
\end{definition}
This definition encompasses many commonly used risk criteria---including the mean, variance, CVaR, and entropic risk measures---under mild boundedness assumptions on the reward support. Detailed examples and their corresponding Lipschitz constants are provided in \Cref{app-lipschitz}. A representative example is the distorted risk functional. For a CDF $F$ supported on $[0,D]$, it is defined as
$$
\rho_g(F) \coloneq \int_{0}^{D} g\bigl(1-F(t)\bigr)\,\drm t,
$$
where $g:[0,1]\mapsto[0,1]$ is a non-decreasing distortion function that reweights tail probabilities. This class unifies several classical risk measures. For instance, taking $g(x)=x$ recovers the expectation, which is $D$-Lipschitz; setting $g(x)=\min\{x/(1-\alpha),1\}$ yields the CVaR at level $\alpha$, which is $D/(1-\alpha)$-Lipschitz. More generally, if $g$ is $L$-Lipschitz, then $\rho_g$ is $DL$-Lipschitz.

\paragraph{Policy class complexity.}
As each policy in $\Pi$ maps contexts to one of $K$ actions, $\Pi$ can be viewed as a $K$-class hypothesis class. We measure its complexity by the Natarajan dimension \citep{natarajan1989learning}, a standard capacity measure for multi-class classification, and assume throughout this work that $\Pi$ has a finite Natarajan dimension $d_\Pi$.
\begin{definition}\label{def-Natarajan-dim}
Given a policy class $\Pi=\{\pi:\Xs\mapsto\As\}$, let $\Ss=\{x_1,\ldots,x_m\}\subset\Xs$ be an arbitrary set of $m$ points. Say that $\Pi$ shatters $\Ss$ if there exists $f_1, f_2: \Ss \mapsto [K]$ such that $f_1(x) \neq f_2(x)$ for all $x\in \Ss$, and for any subset $\mathcal{T}\subset \Ss$, there exists some $\pi\in\Pi$ such that
$$
\pi(x) = 
\begin{cases}
	f_1(x) & \quad \forall x\in \mathcal{T}, \\
	f_2(x) & \quad \forall x\in \Ss \setminus \mathcal{T}.  
\end{cases}
$$
The Natarajan dimension of $\Pi$ is the largest $m$ such that some set of size $m$ is shattered by $\Pi$.%, denoted as $\mathrm{Ndim}(\Pi)$,
\end{definition}

The metric is a natural generalization of the classical Vapnik–Chervonenkis (VC) dimension from binary to multi-class settings, and coincides with the VC dimension when $K = 2$. Finite upper bounds on the Natarajan dimension are available for many commonly used policy classes, including linear function classes, decision trees, random forests, and neural networks \citep{daniely2011multiclass, jin2023natarajan}.

%=======================================================================================
%%%%%%%%%%%%%%%%%%%%%%%%%%%%%%%%%%%%%%%%%%%%%%%%%%%%%%%%%%%%%%%%%%%%%%%%%%%%%%%%%%%%%%%%

\section{Pessimistic Risk-Aware Policy Learning}\label{sec-3}
%%%%%%%%%%%%%%%%%%%%%%%%%%%%%%%%%%%%%%%%%%%%%%%%%%%%%%%%%%%%%%%%%%%%%%%%%%%%%%%%%%%%%%%
%                                  Sec-3
%%%%%%%%%%%%%%%%%%%%%%%%%%%%%%%%%%%%%%%%%%%%%%%%%%%%%%%%%%%%%%%%%%%%%%%%%%%%%%%%%%%%%%%
This section presents the \emph{Pessimistic Risk-Aware Policy Learning} algorithm, which proceeds in two stages: first, it constructs CDF estimations for candidate policies, enabling plug-in evaluation of general risk functionals; second, it adopts a pessimistic policy learning principle and selects a policy based on a conservative assessment of risk performance.

\subsection{Offline Risk-Aware Evaluation}

We begin by introducing an offline risk-aware evaluation framework for assessing the performance of a candidate policy $\pi\in\Pi$. Let $\rho$ be an $L$-Lipschitz risk functional, as defined in \Cref{def-lipschitz-risk}. Given an estimator $\hat F^\pi$ of the marginal reward CDF $F^\pi$, we evaluate the true risk performance $\rho_\pi$ via the corresponding plug-in estimator $$\hat{\rho}_\pi\coloneq\rho(\hat F^\pi).$$By the Lipschitz continuity, the estimation error is directly controlled by the CDF estimation error:
\begin{equation}\label{eq-lip}
\bigl| \hat{\rho}_\pi - \rho_\pi \bigr| \leq L \, \bigr\| \hat F^\pi - F^\pi \bigr\|_\infty.
\end{equation}
Consequently, the problem reduces to constructing a valid estimator $\hat{F}^\pi$ with a guaranteed uniform error bound.

Estimating $F^\pi$ in an off-policy setting is inherently challenging, since rewards are observed only for actions taken by the behavior policy. A classical approach is importance sampling (IS) \citep{precup2000eligibility}, which reweights observed samples to correct this distributional mismatch. Although IS-based CDF estimators have been explored in prior work \citep{huang2021off,chandak2021universal}, their theoretical guarantees rely on the overlap assumption, an constraint we aim to avoid in policy learning setting. In the sequel, we develop an IS estimator that remains valid even in the absence of overlap, with data-dependent concentration bounds provided in \Cref{thm-ope}, and further develop variance-reduced variants, including weighted importance sampling (WIS) and doubly robust (DR) estimators.

\paragraph{IS estimator.}
For any $\pi \in \Pi$ and any $x\in\Xs$, let $\delta_{\pi(x)}$ denote the Dirac measure supported on the action $\pi(x)\in\As$. Suppose that $\delta_{\pi(x)}(\cdot)$ is absolutely continuous with respect to the behavior policy $\beta(x,\cdot)$. A standard change of measure argument then yields
\begin{align*}
G(t\mid x,\pi) & = \Eb_{A\sim \delta_{\pi(x)}} \, G(t\mid x,A) \\
& = \Eb_{A\sim \beta(x,\cdot)} [w_\pi(x,A)  G(t\mid x,A)], \;\forall t\in\Rb,
\end{align*}
where $w_\pi$ is the importance weight of $\pi$, defined as
$$ 
w_\pi (x,a) \coloneq \frac{\1\{ a=\pi(x) \} }{ \beta(x,\pi)},
$$
with $\beta(x,\pi)$ being a shorthand for $\beta(x,\pi(x))$. Since $\delta_{\pi(x)}$ is a point mass, the absolute continuity is equivalent to the pointwise overlap $ \beta(x,\pi)>0$. If this condition additionally holds for $\Pb_X$-almost all $x\in\Xs$, we obtain
$$
F^\pi(t) = \Eb\, G(t\mid X,\pi) = \Eb[w_\pi(X,A)\1\{Y \leq t\}],
$$
where the last expectation is taken with respect to the joint distribution $\Pb_X\times\beta(X,\cdot)\times \Pb_{Y\mid X,A}$, which coincides with the data-generating distribution. Consequently, for each $t\in\Rb$, an unbiased estimator of $F^\pi(t)$ is given by 
$
\frac{1}{n}\sum_{i=1}^{n}w_\pi(X_i,A_i)\1{\{Y_i\leq t\}}.
$

However, when the pointwise overlap condition fails on a context set with positive $\Pb_X$-probability, unbiased estimation of $F^\pi$ from the logged data is no longer possible. In the presence of observations with $\beta(X_i,\pi)=0$, which provide no information about the target policy, classic IS-based estimation of expected values typically resorts to a conservative $-\infty$ assignment \citep{jin2025policy}, thereby discarding other potentially informative samples. We point out in \Cref{rmk-1} that this treatment induces undesirable effects. By contrast, IS-based CDF estimation can still make use of such observations. Let 

\begin{equation}\label{def-Ipi}
\Is_\pi\coloneq\{i\in[n]: \beta(X_i, \pi)\neq 0\}
\end{equation}
be the index set of informative samples for policy $\pi$. For $t\in\Rb$, we define the IS CDF estimator as
\begin{equation}\label{def-is}
\Fis^\pi(t) \coloneq \frac{1}{n}\sum_{i=1}^{n} \hat{G}(t \mid X_i,\pi),
\end{equation}
where
$$
\hat{G}(t \mid X_i,\pi) = 
\begin{cases}
w_\pi(X_i,A_i) \1 \{Y_i \leq t\},& \text{if } i\in\Is_\pi,  \\   % \frac{\1\{ A_i = \pi(X_i)\}}{ \beta(X_i, \pi)}
1,&	\text{if } i\notin\Is_\pi.
\end{cases}
$$

The assignment of the value $1$ for $i \notin \Is_\pi$ is arbitrary and may be replaced by $0$ or any other constant in $[0,1]$. While this choice does not largely affect the subsequent theoretical guarantees, it does influence the inductive bias of the resulting algorithm. In particular, when the risk functional $\rho$ is monotone\footnote{Monotonicity is satisfied by many common risks criteria, including mean, CVaR and more general coherent risk functionals.}, that is for all $t \in \Rb$, 
\begin{equation}\label{def-monotone}
\rho(F_1) \geq \rho(F_2) \quad\text{whenever}\quad F_1(t) \leq F_2(t),
\end{equation}

assigning the value $1$ corresponds to a pessimistic completion of the CDF on unsupported contexts. As a result, maximizing the empirical risk favors policies with larger $|\Is_\pi|$, i.e., policies that remain closer to the behavior policy $\beta$. By contrast, assigning the value $0$ corresponds to a risk-seeking preference, biasing the algorithm toward exploratory policies.

A natural limitation of estimating CDFs via IS is that importance weights may exceed one, in which case the resulting estimator may fail to be a valid CDF. A simple remedy is clipping, leading to the clipped IS estimator
\begin{equation}\label{def-isc}
\Fisc^\pi(t) \coloneq \min \{ \Fis^\pi(t), 1 \}.
\end{equation}
An alternative approach is to normalize the importance weights, which gives rise to the weighted importance sampling (WIS) estimator.

\paragraph{WIS estimator.}
For $t\in\Rb$, we define the WIS CDF estimator as
\begin{equation}\label{def-wis}	
\Fwis^\pi(t) \coloneq \frac{1}{n}\sum_{i=1}^{n} \hat{H}(t \mid X_i,\pi), 
\end{equation}
where
$$
\hat{H}(t \mid X_i,\pi) =  
\begin{cases}
\frac{ 1 }{ W_\pi } w_\pi(X_i,A_i) \1\{Y_i\leq t\}, & \text{if } i\in\Is_\pi,  \\  
1,&	\text{if } i\notin\Is_\pi,
\end{cases}
$$
and $W_\pi$ is the empirical average of importance weights over $\Is_\pi$,
$$
W_\pi \coloneq \frac{1}{|\Is_\pi|}\sum_{i\in\Is_\pi}w_\pi(X_i,A_i).
$$
By construction, $\Fwis^\pi$ is a valid CDF without additional clipping. The normalization, however, introduces bias: even under the overlap assumption, $\Fwis^\pi$ is no longer unbiased, but it remains uniformly consistent \citep{chandak2021universal}, and the bias is typically offset by a reduction in variance, constituting the primary motivation for using WIS.

\paragraph{DR estimator.}
A more commonly used approach for variance reduction is the doubly robust (DR) estimation \citep{dudik2011doubly}. In practice, one often has access to a model $\bar{G}(\cdot \mid x,a)$ of the conditional distribution $G(\cdot\mid x,a)$, for instance obtained by fitting a regression model to the logged data. Such a model can provide meaningful predictions even for context--action pairs poorly covered by the behavior policy. The DR estimator combines this model with importance sampling, and is defined as
\begin{equation}\label{def-dr}
\Fdr^\pi(t) \coloneq \frac{1}{n}\sum_{i=1}^{n} \hat{\Gamma}(t \mid X_i,\pi),
\end{equation}
where $\hat{\Gamma}(t \mid X_i,\pi) =$
\begin{align*}	
\begin{cases}
	\bar{G}(t \mid X_i,\pi) + w_\pi(X_i,A_i) \left[ \1 \{Y_i \leq t\} - \bar{G}(t\mid X_i,A_i) \right],\\
	&\hspace{-5em} \text{if } i\in\Is_\pi, \\ 
	\bar{G}(t \mid X_i,\pi),     		            	&\hspace{-5em} \text{if } i\notin\Is_\pi. 
\end{cases}
\end{align*}
For $i\notin\Is_\pi$, the DR estimator relies entirely on the model-based prediction. For $i\in\Is_\pi$, it augments this prediction with an importance-weighted correction. Crucially, under the overlap condition, $\Fdr^\pi$ remains unbiased even if the model $\bar{G}$ is misspecified, and typically achieves substantially lower variance when $\bar{G}$ provides a reasonable approximation.

Like the IS estimator, $\Fdr^\pi$ is not guaranteed to lie in $[0,1]$, and even may fail to be monotone. We therefore enforce monotonicity and range constraints by defining the clipped and monotonized DR estimator
\begin{equation}\label{def-drc}
\Fdrc^\pi(t) \coloneq \max\Bigl\{ \min\Bigl\{ \max_{t' \leq t}\Fdr^\pi(t'),\, 1 \Bigr\},\, 0 \Bigr\}.
\end{equation}

\subsection{Pessimistic Policy Learning}

The principle of pessimism was proposed to alleviate the uniform overlap assumption in offline learning settings. By regularizing point estimation with an uncertainty penalty, it has been shown to be minimax optimal for expected-reward optimization \citep{jin2025policy}. We incorporate this pessimism principle into risk-aware policy learning. 

Formally, let $\hat{F}^\pi\in\{\Fisc,\Fwis,\Fdrc\}$ be a CDF estimator constructed previously. Suppose we have a policy-dependent confidence bound $R(\pi)$ such that with high probability
$$\| \hat{F}^\pi - F^\pi \|_\infty \leq R(\pi).$$ 
By \Cref{eq-lip}, $\hat{\rho}_\pi - LR(\pi)$ forms a lower confidence bound for the risk $\rho_\pi$. Our algorithm then selects a policy $\tilde{\pi} \in \Pi$ that maximizes this LCB:
\begin{equation}\label{def-pessimism}
\tilde{\pi}  \coloneq \underset{\pi\in\Pi}{\argmax} \quad \hat{\rho}_\pi - LR(\pi).
\end{equation}
The following proposition captures the main benefit of the pessimism principle.
\begin{proposition}\label{pro-pessi}
On the event
\begin{equation}\label{eq-R}
	\{ \| \hat{F}^\pi - F^\pi \|_\infty  \leq R(\pi), \forall \pi \in \Pi \},
\end{equation}
the suboptimality of $\tilde{\pi}$ is bounded by
$$
\rho_{\pi^*} - \rho_{\tilde{\pi}} \leq 2 LR(\pi^*).
$$
\end{proposition}

\begin{proof}
On the stated event, the $L$-Lipschitz continuity implies $| \rho_{\pi^*} - \hat{\rho}_{\pi^*} | \leq LR(\pi^*) $ and $ | \rho_{\tilde{\pi}} - \hat{\rho}_{\tilde{\pi}} | \leq LR(\tilde{\pi}) $. Together with the pessimistic optimality $ \hat{\rho}_{\tilde{\pi}} - LR(\tilde{\pi}) \geq \hat{\rho}_{\pi^*} - LR(\pi^*) $, it holds
\begin{align*}
	\rho_{\pi^*} - \rho_{\tilde{\pi}} & = \rho_{\pi^*} - \hat{\rho}_{\pi^*} + \hat{\rho}_{\pi^*}- \hat{\rho}_{\tilde{\pi}} + \hat{\rho}_{\tilde{\pi}} - \rho_{\tilde{\pi}} \\ 
	& \leq LR(\pi^*) +  \hat{\rho}_{\pi^*} - LR(\pi^*) -  \hat{\rho}_{\tilde{\pi}} + LR(\tilde{\pi}) \\
	& \qquad  + \hat{\rho}_{\tilde{\pi}} - LR(\tilde{\pi})  - \rho_{\tilde{\pi}} +  LR(\pi^*) \\
	& \leq 2LR(\pi^*). 
\end{align*}
The claim follows.
\end{proof}

\Cref{pro-pessi} establishes that for any $L$-Lipschitz risk functional, the suboptimality bound of $\tilde{\pi}$ depends only on $R(\pi^*)$, the CDF estimation error bound of the optimal policy, regardless of the estimation errors of other policies. It thus remains to construct a suitable $R(\pi)$ with uniform concentration guarantees over $\Pi$, as in \Cref{eq-R}.

%=======================================================================================
%%%%%%%%%%%%%%%%%%%%%%%%%%%%%%%%%%%%%%%%%%%%%%%%%%%%%%%%%%%%%%%%%%%%%%%%%%%%%%%%%%%%%%%%

\section{Theoretical Results}\label{sec-4}
%%%%%%%%%%%%%%%%%%%%%%%%%%%%%%%%%%%%%%%%%%%%%%%%%%%%%%%%%%%%%%%%%%%%%%%%%%%%%%%%%%%%%%%
%                                  Sec-4
%%%%%%%%%%%%%%%%%%%%%%%%%%%%%%%%%%%%%%%%%%%%%%%%%%%%%%%%%%%%%%%%%%%%%%%%%%%%%%%%%%%%%%%%
This section presents the theoretical contributions of this work. We first establish data-dependent concentration bounds for a fixed policy without overlap. Building on this, we derive uniform concentration results over the policy class $\Pi$ in the form of \Cref{eq-R} for different CDF estimators introduced earlier, and finally show that the resulting bounds are minimax optimal.

\subsection{Concentration Inequality for a Single Policy}

In general, confidence bounds for IS-based estimators depend on either the worst-case importance weight $\sup w(x,a)$ or its second moment $\mathbb{E}[w(X,A)^2]$, corresponding respectively to Hoeffding-type and Bernstein-type concentration inequalities. To handle policies without overlap, for which these quantities are no longer valid, we seek for surrogates that are data-dependent. Leveraging the index set $\Is_\pi$ defined in \Cref{def-Ipi}, we define
\begin{equation}\label{def-sigma}
\sigma_\pi \coloneq \sqrt{\frac{1}{n}\sum_{i\in\Is_\pi}\frac{1}{\beta(X_i, \pi)^2}}
\end{equation}
and
\begin{equation}\label{def-sigma'}
\sigma_\pi' \coloneq \sqrt{\frac{1}{n}\sum_{i\in\Is_\pi}\frac{1}{\beta(X_i, \pi)}},
\end{equation}
which serve as proxies for $\sup w(x,a)$ and $\sqrt{\mathbb{E}[w(X,A)^2]}$, respectively. In addition, let
\begin{equation}\label{def-r}
r_\pi \coloneq \frac{n-|\Is_\pi|}{n}
\end{equation}
denote the fraction of uninformative samples for evaluating $\pi$.
Based on these quantities, the following theorem establishes data-dependent high-confidence upper bounds on the sup-norm error of the CDF estimator for a fixed policy.

\begin{theorem}\label{thm-ope}
Fix any policy $\pi \in \Pi$. For the clipped IS estimator $\Fisc^\pi$ defined in \Cref{def-isc} and $\sigma_\pi,\sigma_\pi', r_\pi$ defined in \Cref{def-sigma,def-sigma',def-r}, it holds for any $\delta\in(0,1)$ that
$$
\Pb \left( \left\| \Fisc^\pi - F^\pi \right\|_\infty \leq  \left( \sigma_\pi + 1  \right)  \sqrt{\frac{8}{n}\log\frac{8}{\delta}}  + r_\pi \right) \geq 1-\delta.
$$
In addition, a Bernstein-type bound also holds: with probability at least $1-\delta$,
$$
\left\| \Fisc^\pi - F^\pi \right\|_\infty \leq  \left( \sigma_\pi' + 1  \right)  \sqrt{\frac{8}{n}\log\frac{8}{\delta}} + \frac{2\log(8/\delta)}{3n \beta_{\min}}  + r_\pi,
$$
where $\beta_{\min}\coloneq\min_{i\in\Is_\pi}\beta(X_i, \pi)$.
\end{theorem}

The proof is deferred to \Cref{pf-thm-ope}. 
Both bounds decompose into a concentration-controlled stochastic deviation term and an inevitable bias term $r_\pi$ arising from lack of overlap. 
In terms of the stochastic deviation, the Bernstein-type bound is typically sharper than the Hoeffding-type bound, as the coefficient of the $n^{-1/2}$ term, $\sigma_\pi'$, is generally smaller than $\sigma_\pi$. Note that under an overlap condition for $\pi$, we have $\mathcal I_\pi = [n]$ and hence the bias term $r_\pi = 0$. In this case, our bounds achieve the same $\mathcal{O}(n^{-1/2})$ convergence rate as in \citet{huang2021off}. However, the data-dependent constant $\sigma_\pi + 1$ typically improves upon the worst-case importance weight factor $1 / \inf_x \beta(x,\pi)$ appearing in prior work. To the best of our knowledge, this result provides the first empirical DKW-style concentration inequality for the IS-based CDF estimator.

We present the result only for the IS estimator, which is a building block for the other IS-based variants. Corresponding bounds for the rest can be found in the proofs of their uniform concentration analyses below.

\subsection{Uniform Concentration Inequalities}

We next extend the concentration results for a fixed policy to uniform guarantees over the entire policy class, constructing data-dependent policy-dependent confidence bound $R(\pi)$ satisfying \Cref{eq-R} for each estimator.

\begin{theorem}\label{thm-uniBound-IS}
Let $\Fisc^\pi$ be the IS estimator defined in \Cref{def-isc}, with $\sigma_\pi, r_\pi$ defined in \Cref{def-sigma}, \Cref{def-r}. For any $\delta\in(0,1)$, define
$$
R(\pi) \coloneq ( \sigma_\pi + 2 ) \sqrt{\frac{8}{n} \left[  \log\frac{20}{\delta} +  d_{\Pi}  \log(nK^2) \right]  }  +  r_\pi.
$$
Then, it holds
$$
\Pb \left( \left\| \Fisc^\pi - F^\pi \right\|_\infty \leq R(\pi), \forall \pi\in\Pi \right) \geq 1-\delta.
$$
\end{theorem}

The proof is deferred to \Cref{pf-thm-uniBound-IS}. Compared with the pointwise bound in \Cref{thm-ope}, $R(\pi)$ incurs an additional complexity term $\sqrt{d_{\Pi}}$ to ensure uniformity over $\Pi$. Moreover, since \Cref{pro-pessi} shows that the suboptimality of $\tilde{\pi}$ depends only on $R(\pi^*)$, we obtain a convergence rate of $\tilde{\mathcal{O}}(n^{-1/2})$ under the overlap condition for $\pi^*$, where $\tilde{\mathcal{O}}(\cdot)$ omits logarithm factors. This rate matches that achieved for expected-reward maximization \citep{jin2025policy}.

\begin{remark}\label{rmk-1}
We emphasize that our distributional approach should be preferable to expectation-based pessimistic learning \citep{jin2025policy}, even when the ultimate objective is expected-reward maximization. 
Recalling the pessimistic policy optimization procedure in \Cref{def-pessimism}, the penalty term $L R(\pi)$ remains meaningful even for policies without overlap. 
Specifically, $R(\pi)$ consists of two components: the fraction of uninformative samples $r_\pi$, and the uncertainty penalty of informative samples, characterized by $\sigma_\pi$ or $\sigma_\pi'$. 
As a result, even when no policy in $\Pi$ satisfies the overlap condition, our algorithm can still identify policies that maximize the conservative estimate on informative samples while having the fewest uninformative ones, along with meaningful suboptimality guarantees. 
In contrast, \citet{jin2025policy} assigns an infinite penalty uniformly to all policies without overlap, thereby restricting optimization to the overlap-satisfying policies and providing no suboptimality guarantee when the optimal policy itself lacks overlap. 
We note that this key difference arises from the fact that our distributional estimators preserve information from all samples, including those inevitably discarded by estimators targeting a single functional, such as the expectation.
\end{remark}

For clarity of presentation, the results in this and subsequent theorems are driven based on the first concentration bound in \Cref{thm-ope}. Analogous results can be obtained using the Bernstein-type bound, in which case $\sigma_\pi$ is replaced by $\sigma_\pi'$ and an additional $n^{-1}$ term is included.

\begin{theorem}\label{thm-uniBound-WIS}
Let $\Fwis^\pi$ be the WIS estimator defined in \Cref{def-wis}, with $\sigma_\pi, r_\pi$ defined in \Cref{def-sigma}, \Cref{def-r}. For any $\delta\in(0,1)$, define $R(\pi) \coloneq \1\{ \eta_{\pi} \geq 1 \} +  \1\{ \eta_{\pi} < 1 \} \cdot \xi_{\pi}$, where
\begin{align*}
	&\xi_{\pi} \coloneq (\frac{ \sigma_\pi }{ 1-\eta_{\pi} } + 2) \sqrt{\frac{8}{n}[\log\frac{20}{\delta}+ d_{\Pi}  \log(nK^2) ]} \\
	& \hspace*{10em} +  \frac{|\Is_\pi|}{n}\frac{\eta_{\pi}}{1-\eta_{\pi}} + r_\pi,
\end{align*}
and $\eta_{\pi} \coloneq \sigma_\pi\sqrt{ \frac{n}{2|\Is_\pi|^2} [\log\frac{8}{\delta} +  d_{\Pi}  \log(nK^2)] }$. It then holds
$$
\Pb \left( 	\left\| \Fwis^\pi - F^\pi \right\|_\infty \leq R(\pi), \; \forall \pi\in\Pi \right) \geq 1-\delta.
$$
\end{theorem}

The proof is deferred to \Cref{pf-thm-uniBound-WIS}. The key idea in analyzing WIS is using $\eta_\pi$ to quantify the deviation of the averaged importance weights $W_\pi$ from 1. When $\eta_\pi \ge 1$, we conservatively bound the CDF error by 1. However, we typically have $\eta_{\pi} < 1$, so that $R(\pi) = \xi_\pi$. In fact, for policies such that the pointwise overlap holds with a positive $\Pb_X$-probability, we have $|\Is_\pi| \asymp n$, implying $\eta_\pi = \tilde{\mathcal{O}}(n^{-1/2})$. In this regime, the resulting convergence rate matches that of the IS estimator in \Cref{thm-uniBound-IS}. To the best of our knowledge, this provides the first finite-sample guarantee for the WIS CDF estimator.

\begin{theorem}\label{thm-uniBound-DR}
Let $\Fdrc^\pi$ be the IS estimator defined in \Cref{def-drc}, with $\sigma_\pi$ defined in \Cref{def-sigma}. For any $\delta\in(0,1)$, define
$$
R(\pi) \coloneq 2(\sigma_\pi + 1) \sqrt{ \frac{8}{n} \left[ \log\frac{20}{\delta} + d_{\Pi} \log(nK^2) \right] }  +  \bar{r}_\pi,
$$
where $\bar{r}_\pi \coloneq \left\| \frac{1}{n}\sum_{i\notin\Is_\pi} \left[ \bar{G}(\cdot\mid X_i,\pi) - G(\cdot\mid X_i,\pi) \right] \right\|_\infty$.
It then holds that
$$
\Pb \left( 	\left\| \Fdrc^\pi - F^\pi \right\|_\infty \leq R(\pi), \; \forall \pi\in\Pi \right) \geq 1-\delta.
$$
\end{theorem}

The proof is deferred to \Cref{pf-thm-uniBound-DR}. 
By incorporating the model $\bar{G}$, the bias due to lack of overlap is captured by $\bar{r}_\pi$, which can be substantially smaller than $r_\pi$ if the model is accurate. For samples in $\Is_\pi$, the DR estimator achieves a similar convergence rate up to a constant factor relative to the IS estimator. 
Nevertheless, the bound is conservative since it accounts for worst-case model misspecification, under which the DR estimator may perform worse than IS. 
Overall, $\Fdrc^\pi$ inherits the advantages of both IS and the model, resulting in more robust performance in practice.

Combining \Cref{thm-uniBound-IS}--\ref{thm-uniBound-DR} with Proposition~\ref{pro-pessi} immediately yields high-confidence upper bounds on the suboptimality of $\tilde{\pi}$ for various CDF estimators. Here, we illustrate the result using the DR estimator, under a overlap condition for the optimal policy.

\begin{corollary}\label{cor-bound-dr}
Suppose there exists a constant $\beta_{\inf}>0$ such that $\inf_{x\in\Xs}\beta(x,\pi^*)\geq \beta_{\inf}$, and that the sample size $n$ satisfies  $\log(20/\delta) +  d_{\Pi} \log(nK^2) \leq c_0 n\beta_{\inf}$. Then, for any constant $c \geq 8 (4\sqrt{2} + \frac{\sqrt{c_0}}{3})$ and any $\delta\in(0,1)$, it holds with probability at least $1-\delta$ that
$$
\rho_{\pi^*} - \rho_{\tilde{\pi}} \leq c L  \sqrt{ \frac{  d_{\Pi} \log(nK^2) \cdot \log(20/\delta) }{n\beta_{\inf}} } .
$$
\end{corollary}

The proof is deferred to \Cref{pf-cor-bound-dr}. This corollary implies that provably efficient offline policy learning with an $\tilde{\mathcal{O}}(n^{-1/2})$ convergence rate is feasible for all Lipschitz risk functional, as long as the optimal policy is sufficiently explored. Moreover, the suboptimality bound scales with the sample size $n$, the policy class complexity $d_\Pi$, and the overlap constant $\beta_{\inf}$ in the same order as in the expected-reward optimization \citep{jin2025policy}, indicating that offline risk-aware policy learning is statistically no harder than conventional expected-reward learning.

\subsection{Minimax Optimality}

We now establish a minimax lower bound, trying to show that the convergence rate obtained in \Cref{cor-bound-dr} is optimal given the policy $\Pi$. We first make the dependence on the underlying environment explicit. For any contextual bandit environment $v$, let $F^{\pi,v}$ denote the marginal reward CDF induced by policy $\pi$ under $v$, and let 
$$\pi^*_v \in \argmax_{\pi\in\Pi}\; \rho(F^{\pi,v})$$
be the optimal policy of $v$ for a fixed risk functional $\rho$. We consider $n$ offline samples generated from $v$ using a behavior policy $\beta$, and restrict attention to instances for which the optimal policy satisfies a overlap condition as in \Cref{cor-bound-dr}.

\begin{theorem}\label{thm-minimax}
Suppose the risk measure $\rho$ is monotone as defined in \Cref{def-monotone}. For a constant $\beta_{\inf}\in(0,1)$, denote by $\Vs$ the class of environment and behavior policy pairs for which the optimal policy has a $\beta_{\inf}$ overlap, i.e.,
$$\Vs \coloneq \{ (v,\beta): \inf_{x\in\Xs} \beta(x,\pi^*_v)\geq\beta_{\inf} \}.$$
If $n\beta_{\inf}/d_{\Pi} \geq 2/3$, then for any (possibly data-dependent) policy $\hat{\pi}\in\Pi$, it holds that
$$
\inf_{\hat{\pi}\in\Pi} \sup_{(v,\beta)\in\Vs} \Eb\,W_1( F^{\pi^*_v,v} , F^{\hat{\pi},v} ) \geq \frac{1}{8\sqrt{6e}}\sqrt{ \frac{d_\Pi}{n\beta_{\inf}} },
$$
where $W_1(\cdot,\cdot)$ denotes the $1$-Wasserstein distance of two CDFs.
\end{theorem}

The proof is deferred to \Cref{pf-thm-minimax}. 
Theorem~\ref{thm-minimax} establishes a minimax lower bound measured in the $1$-Wasserstein distance \citep{villani2009the}, which is deliberately chosen here. Note that under the monotonicity assumption \Cref{def-monotone}, as $\pi^*_v$ is the optimal policy of $v$, we have
$$ F^{\pi^*_v,v}(t) \leq  F^{\hat{\pi},v}(t)\;\text{ for all }\;t\in\Rb.$$ 
This ordering enables a reverse form of Lipschitz continuity for a broad class of risk measures. Specifically, there exists a constant $L'>0$ such that
\begin{equation}\label{def-lip-rev}
\rho(F^{\pi^*_v,v}) - \rho(F^{\hat{\pi},v}) \geq L' W_1(F^{\pi^*_v,v}, F^{\hat{\pi},v} ).
\end{equation}
Under mild regularity conditions, \eqref{def-lip-rev} is satisfied by spectral risk measures, which include the mean, CVaR, and other quantile-based criteria.
Consequently, combining \Cref{def-lip-rev} with Theorem~\ref{thm-minimax} yields
$$\inf_{\hat{\pi}\in\Pi} \sup_{(v,\beta)\in\Vs} \Eb [\rho(F^{\pi^*_v,v}) -\rho(F^{\hat{\pi},v}) ] \geq \frac{L'}{8\sqrt{6e}}\sqrt{ \frac{d_\Pi}{n\beta_{\inf}} }.$$
This lower bound matches the upper bound in \Cref{cor-bound-dr} up to logarithmic factors, showing that the learned policy $\tilde{\pi}$ is minimax optimal for any risk measure satisfying \eqref{def-lip-rev}.

%=======================================================================================
%%%%%%%%%%%%%%%%%%%%%%%%%%%%%%%%%%%%%%%%%%%%%%%%%%%%%%%%%%%%%%%%%%%%%%%%%%%%%%%%%%%%%%%%

\section{Conclusion}\label{sec-5}
%%%%%%%%%%%%%%%%%%%%%%%%%%%%%%%%%%%%%%%%%%%%%%%%%%%%%%%%%%%%%%%%%%%%%%%%%%%%%%%%%%%%%%%
%                                  Sec-5
%%%%%%%%%%%%%%%%%%%%%%%%%%%%%%%%%%%%%%%%%%%%%%%%%%%%%%%%%%%%%%%%%%%%%%%%%%%%%%%%%%%%%%%%

In this paper, we propose a novel distributional approach for offline policy learning under general Lipschitz-continuous risk functionals in contextual bandits. Leveraging the pessimistic principle, our method avoids the uniform overlap requirement and achieves minimax optimal suboptimality rate under minimal assumptions. Theoretically, for common IS-based CDF estimators, we establish the first empirical DKW-type concentration inequalities that hold with unbounded importance weights, which are expected to have broad applicability beyond this work.

A notable limitation of our analysis is that, although the theoretical guarantees holds for unbounded rewards, risk functionals are sup-norm Lipschitz continuous only when CDFs have bounded supports. Extending the analysis to other metrics, such as the 1-Wasserstein distance, where many risk functionals remain Lipschitz for unbounded rewards, is nontrivial and left for future work. Moreover, we focus exclusively on the contextual bandits. Extending our framework to sequential decision making problems, i.e., risk-sensitive offline RL, constitutes an important direction for future research.

%---------------------------------------------------------------------------------------

%%%%%%%%%%%%%%%%%%%%%%%%%%%%%%%%%%%%%%%%%%%%%%%%%%%%%%%%%%%%%%%%%%%%%%%%%%%%%%%%%%%%%%%%
%                                 Appendix
%%%%%%%%%%%%%%%%%%%%%%%%%%%%%%%%%%%%%%%%%%%%%%%%%%%%%%%%%%%%%%%%%%%%%%%%%%%%%%%%%%%%%%%%

\appendix
\newpage

\section{Common Risk Functionals and their Lipschitz Continuity}\label{app-lipschitz}
%%%%%%%%%%%%%%%%%%%%%%%%%%%%%%%%%%%%%%%%%%%%%%%%%%%%%%%%%%%%%%%%%%%%%%%%%%%%%%%%%%%%%%%%
%                                  App-A
%%%%%%%%%%%%%%%%%%%%%%%%%%%%%%%%%%%%%%%%%%%%%%%%%%%%%%%%%%%%%%%%%%%%%%%%%%%%%%%%%%%%%%%%

This appendix reviews several commonly used risk functionals and establish their Lipschitz continuity with respect to the sup norm on CDFs, as defined in \Cref{def-lipschitz-risk}. Let $X$ be a real-valued random variable with CDF $F$. Formally, a risk functional $\rho$ is a mapping from the space of random variables to $\Rb$. We restrict attention to law-invariant risk functionals \citep{kusuoka2001on}, meaning that $\rho(X)=\rho(X')$ whenever $X$ and $X'$ share the same distribution. Consequently, we equivalently treat $\rho$ as a function of distributions in the main text. With a slight abuse of notation, we use $\rho(X)$ and $\rho(F)$ interchangeably.

\paragraph{Mean--Variance.}
An effective approach to balancing expected reward and variability is the mean--variance (MV) risk functional, defined as
$$
\mathrm{MV}_\alpha(X) \coloneq \Eb[X] + \alpha \Var(X),
$$
where $\alpha \in \Rb$ controls the agent’s risk preference and sensitivity. Specifically, a negative $\alpha$ corresponds to a risk-averse objective (prioritizing stability over expected value), a positive $\alpha$ denotes a risk-seeking preference (favoring high uncertainty), and a large $|\alpha|$ implies stronger risk sensitivity.
\begin{lemma}
On the space of distributions supported on $[0,D]$, the mean is $D$-Lipschitz continuous, the variance is $3D^2$-Lipschitz continuous, and the mean--variance functional $\mathrm{MV}_\alpha$ is $(D + 3|\alpha|D^2)$-Lipschitz.
\end{lemma}
\begin{proof} 
Let  $F$ and $F'$ be the CDFs of any two random variables $X, X'\in[0,D]$. Using the identity $\Eb X = \int_{0}^{D}[1-F(t)]\drm t$, we obtain $| \Eb X - \Eb X' | = \left| \int_{0}^{D}[F'(t)-F(t)]\drm t \right| \leq D\cdot\|F-F'\|_\infty$. For  the variance, note that $\Eb X^2 = \int_{0}^{D^2}\Pb(X^2>t)\drm t =  \int_{0}^{D}2t\Pb(X>t)\drm t$, we have $\Var(X) = \int_{0}^{D}2t[1-F(t)] \drm t -  \left(\int_{0}^{D}[1-F(t)]\drm t\right)^2$. Therefore,
\begin{align*}
	|\Var(X)-\Var(X')| & \leq  \left| \int_{0}^{D}2t[F'(t)-F(t)]\drm t \right| + \left| (\int_{0}^{D}[1-F(t)]\drm t)^2 - (\int_{0}^{D}[1-F'(t)]\drm t)^2 \right| \\
	& \leq D^2\|F-F'\|_\infty + \left| \int_{0}^{D}[1-F(t)]\drm t + \int_{0}^{D}[1-F'(t)]\drm t \right| \cdot \left| \int_{0}^{D}[F'(t)-F(t)]\drm t \right|  \\
	& \leq D^2\|F-F'\|_\infty + 2D \cdot D\|F-F'\|_\infty \\
	& = 3D^2 \|F-F'\|_\infty.
\end{align*}
For the mean-variance, $|\mathrm{MV}_\alpha(X)-\mathrm{MV}_\alpha(X')| \leq |\Eb X - \Eb X'| + |\alpha||\Var(X)-\Var(X')| \leq (D+3|\alpha|D^2) \|F-F'\|_\infty.$
\end{proof}

\paragraph{Entropic risk functional.}
The entropic risk functional arises from exponential utility, capturing sensitivity to tail outcomes as a smooth, convex risk measure. It is defined as
$$
\mathrm{Ent}_\alpha(X) \coloneq \frac{1}{\alpha}\log \left(\Eb e^{\alpha X} \right),
$$
where $\alpha\neq 0$ is a scalar parameter. A second-order Taylor expansion yields
$
\mathrm{Ent}_\alpha(X)
= \Eb[X] + \frac{\alpha}{2}\Var(X) + O(\alpha^2),
$
indicating that $\alpha$ controls risk preference and sensitivity analogous to the mean--variance risk measure.
\begin{lemma}
On the space of distributions supported on $[0,D]$, the entropic risk functional
$\mathrm{Ent}_\alpha(\cdot)$ is $ \frac{1}{\alpha}(e^{\alpha D}-1)$-Lipschitz if $\alpha>0$, and $\frac{1}{\alpha e^{\alpha D}}(e^{\alpha D}-1)$-Lipschitz if $\alpha<0$.
\end{lemma}
\begin{proof}
First consider $\alpha>0$. Using the identity
\begin{equation*}
	\Eb e^{\alpha X}  = \int_{0}^{+\infty} \Pb(e^{\alpha X}>t) \drm t = \int_{-\infty}^{+\infty} e^t\Pb(\alpha X>t)\drm t = \int_{-\infty}^{0} e^t \drm t + \int_{0}^{\alpha D} e^t[1-F(t/\alpha)]\drm t,
\end{equation*}
we have
$$
|\mathrm{Ent}_\alpha(X) - \mathrm{Ent}_\alpha(X')| \leq \frac{1}{\alpha} |\Eb e^{\alpha X} - \Eb e^{\alpha X'}| \leq  \frac{1}{\alpha}(e^{\alpha D}-1)\|F-F'\|_\infty. 
$$
The case $\alpha<0$ follows analogously and is omitted.
\end{proof}

\paragraph{Conditional value-at-risk.}
For a random variable $X$, the $\mathrm{CVaR}$ at level $\alpha\in(0,1)$ is defined as
$$
\mathrm{CVaR}_\alpha(X) \coloneq \inf_{\xi\in\Rb} \left\{ \xi+ \frac{1}{1-\alpha}\Eb (X-\xi)^+ \right\},
$$
where $x^+\coloneq \max\{x,0\}$. It is well known \citep{rockafellar2000optimization} that the infimum above is attained by $\xi = \mathrm{VaR}_\alpha(X) \coloneq \inf \{ \xi : \Pb(X\leq \xi ) \geq \alpha \}$, which is the Value-at-Risk (VaR) of $X$ at confidence level $\alpha$, also referred to as the $\alpha$-quantile. Therefore, CVaR also admits the following equivalent representation:
$$
\mathrm{CVaR}_\alpha(X) = \mathrm{VaR}_\alpha(X) + \frac{1}{1-\alpha}\Eb [X-\mathrm{VaR}_\alpha(X)]^+ .
$$
It can be seen that $\mathrm{CVaR}_\alpha(X)$ corresponds to the expectation of $X$ conditional on its upper $(1-\alpha)$ tail. As $\alpha\to0$, it reduces to the full expectation $\Eb X$; as  $\alpha\to1$, it recovers the essential supremum of $X$. A key advantage of CVaR over VaR is that CVaR is a coherent risk measure, while VaR is not \citep{artzner1999coherent}.
\begin{lemma}
On the space of distributions supported on $[0,D]$, $\mathrm{CVaR}_\alpha$ is a $ \frac{D}{1-\alpha}$-Lipschitz risk functional.
\end{lemma}
\begin{proof}
By definition, we know that $F(t)< \alpha$ for all $t< \mathrm{VaR}_\alpha(X)$. This gives
\begin{align*}
	\int_{0}^{D} \min\{\frac{1-F(t)}{1-\alpha}, 1\} \drm t & = \frac{1}{1-\alpha} \int_{0}^{D} \min\{1-F(t), 1-\alpha\} \drm t\\
	& =  \frac{1}{1-\alpha} \int_{0}^{ \mathrm{VaR}_\alpha(X) }  (1-\alpha) \drm t + \frac{1}{1-\alpha} \int_{\mathrm{VaR}_\alpha(X)}^{ D }  [1-F(t)] \drm t \\
	& = \mathrm{VaR}_\alpha(X) + \int_{0}^{ D-\mathrm{VaR}_\alpha(X) }  \Pb(X-\mathrm{VaR}_\alpha(X)>t) \drm t \\
	& = \mathrm{VaR}_\alpha(X) + \frac{1}{1-\alpha}\Eb [X-\mathrm{VaR}_\alpha(X)]^+ = \mathrm{CVaR}_\alpha(X).
\end{align*}
Hence, $|\mathrm{CVaR}_\alpha(X) - \mathrm{CVaR}_\alpha(X')| \leq D \sup_t |\min\{\frac{1-F(t)}{1-\alpha}, 1\} - \min\{\frac{1-F'(t)}{1-\alpha}, 1\}| \leq D \frac{1}{1-\alpha} \|F-F'\|_\infty $.
\end{proof}

\paragraph{Distorted risk functionals.}
For a non-negative random variable $X$ with CDF $F$, a distortion risk measure is defined through a distortion function $g:[0,1]\mapsto[0,1]$ that is non-decreasing and satisfies $ g(0)=0, g(1)=1 $: 
$$
\rho_g(F) \coloneq \int_{0}^{+\infty} g(1-F(t)) \drm t.
$$
Distorted risk functionals constitute a broad class, unifying many classical risk measures. For example, setting $g(x)=\min\{x/(1-\alpha),1\}$ yields $\mathrm{CVaR}_\alpha(X)$ as shown above, choosing $g(x)=\mathbbm{1}\{x\ge 1-\alpha\}$ recovers $\mathrm{VaR}_\alpha(X)$, and taking $g(x)=x$ reduces $\rho_g$ to the expectation $\Eb[X]$. The distortion function $g$ governs the agent's preference by reweighting the tail probabilities. Concave distortion functions overweight tail events relative to the linear case, while convex distortions attenuate tail contributions. The Lipschitz continuity of $\rho_g$ follows directly from the distortion function.
\begin{lemma}
On the space of distributions supported on $[0,D]$, a distorted risk functional $\rho_g$ is $ DL$-Lipschitz if the distortion function $g$ is $L$-Lipschitz.
\end{lemma}

\paragraph{Cumulative prospect theory (CPT) risk functionals.}
Indexed by $u=(u^+,u^-), w=(w^+,w^-)$, the CPT risk measure of a random variable $X$ is defined as
$$
\mathrm{CPT}_{u,w}(X) \coloneq \int_{0}^{+\infty} w^+(\Pb(u^+(X)>t)) \drm t - \int_{0}^{+\infty} w^-(\Pb(u^-(X)>t)) \drm t,
$$
where $u^+,u^-:\Rb\mapsto\Rb_+$ are continuous, with $u^+(x)=0$ when $x\leq0$ and increasing otherwise, $u^-(x)=0$ when $x\geq0$ and decreasing otherwise; $w^+,w^-:[0,1]\mapsto[0,1]$ are continuous, with $w^+(0)=w^-(0)=0$ and $w^+(1)=w^-(1)=1$. CPT generalizes distorted risk functionals by handling gains and losses separately \citep{prashanth2016cumulative}. 
\begin{lemma}
On the space of distributions supported on $[0,D]$, a CPT risk functionals $\mathrm{CPT}_{u,w}$ is $u^+(D)L$-Lipschitz if the gain-side weighting function $w^+$ is $L$-Lipschitz.
\end{lemma}

%=======================================================================================
%%%%%%%%%%%%%%%%%%%%%%%%%%%%%%%%%%%%%%%%%%%%%%%%%%%%%%%%%%%%%%%%%%%%%%%%%%%%%%%%%%%%%%%%

\section{Proofs for \Cref{sec-4}}
%%%%%%%%%%%%%%%%%%%%%%%%%%%%%%%%%%%%%%%%%%%%%%%%%%%%%%%%%%%%%%%%%%%%%%%%%%%%%%%%%%%%%%%%
%                                  App-B
%%%%%%%%%%%%%%%%%%%%%%%%%%%%%%%%%%%%%%%%%%%%%%%%%%%%%%%%%%%%%%%%%%%%%%%%%%%%%%%%%%%%%%%%
\subsection{Proof of \Cref{thm-ope}}\label{pf-thm-ope}

\begin{proof}
	
Recall the definition \Cref{def-is} of $\Fis^\pi$. For any fixed $\pi\in \Pi$, we have the following decomposition: \footnote{For brevity, all parentheses following $\sum$ are omitted throughout the proofs.}
$$
\left\| \Fis^\pi - F^\pi \right\|_\infty \leq \left\| \frac{1}{n}\sum_{i=1}^{n}  \hat{G}(\cdot\mid X_i,\pi) - G(\cdot\mid X_i,\pi) \right\|_\infty  + \left\|\frac{1}{n}\sum_{i=1}^{n} G(\cdot\mid X_i,\pi) - F^\pi(\cdot) \right\|_\infty.
$$ 

For the first term, conditioning on $\{X_i\}_{i=1}^n$, we further decompose it as 
\begin{align*}
	& \left\| \frac{1}{n}\sum_{i=1}^{n}\hat{G}(\cdot\mid X_i,\pi) - G(\cdot\mid X_i,\pi) \right\|_\infty \\ 
	&\leq \left\| \frac{1}{n}\sum_{i\in\Is_\pi}\hat{G}(\cdot\mid X_i,\pi) - G(\cdot\mid X_i,\pi) \right\|_\infty +  \left\| \frac{1}{n}\sum_{i\notin\Is_\pi}\hat{G}(\cdot\mid X_i,\pi) - G(\cdot\mid X_i,\pi) \right\|_\infty \\
	&= \left\| \frac{1}{n}\sum_{i\in\Is_\pi}w_\pi(X_i,A_i)\1\{Y_i\leq \cdot\} - G(\cdot\mid X_i,\pi) \right\|_\infty +  \left\| \frac{1}{n}\sum_{i\notin\Is_\pi}1 - G(\cdot\mid X_i,\pi) \right\|_\infty \\
	&= \left\| \frac{1}{n}\sum_{i\in\Is_\pi}w_\pi(X_i,A_i)\1\{Y_i\leq \cdot\} - G(\cdot\mid X_i,\pi) \right\|_\infty +  r_\pi.
\end{align*}
Applying the first confidence bound from \Cref{lem-ope-1} then gives
\begin{equation}\label{pf-eq-a4}
	\Pb\left( \left\| \frac{1}{n}\sum_{i=1}^{n}\hat{G}(\cdot\mid X_i,\pi) - G(\cdot\mid X_i,\pi) \right\|_\infty \geq  \sigma_\pi  \sqrt{\frac{8}{n}\log\frac{8}{\delta}}  + r_\pi \;\Big|\; \{X_i\}_{i=1}^n \right) \leq \frac{\delta}{2}.
\end{equation}

Marginalizing over $\{X_i\}_{i=1}^n$ and applying \Cref{lem-ope-2}, which controls the second term, we have
$$
\Pb \left( \left\| \Fis^\pi - F^\pi \right\|_\infty \geq  \left( \sigma_\pi + 1  \right)  \sqrt{\frac{8}{n}\log\frac{8}{\delta}}  + r_\pi \right) \leq \delta.
$$
Together with the fact $\left\| \Fisc^\pi - F^\pi \right\|_\infty  \leq \left\| \Fis^\pi - F^\pi \right\|_\infty$, the first result holds. The second result follows similarly. We complete the proof.

\end{proof}

%%%%%%%%%%%%%%%%%%%%%%%%%%%%%%%%%%%%%%%%%%%%

\begin{lemma}\label{lem-ope-1}
For any $\delta\in(0,1)$, it holds that
$$
\Pb\left( \left\| \frac{1}{n}\sum_{i\in\Is_\pi}w_\pi(X_i,A_i)\1\{Y_i\leq \cdot\} - G(\cdot\mid X_i,\pi) \right\|_\infty \geq \sigma_\pi \sqrt{\frac{8}{n}\log\frac{8}{\delta}} \;\Big|\; \{X_i\}_{i=1}^n \right) \leq \frac{\delta}{2}.
$$
Moreover, a Bernstein-type inequality holds:
$$
\Pb\left( \left\| \frac{1}{n}\sum_{i\in\Is_\pi}w_\pi(X_i,A_i)\1\{Y_i\leq \cdot\} - G(\cdot\mid X_i,\pi) \right\|_\infty \geq \sigma_\pi' \sqrt{\frac{8}{n}\log\frac{8}{\delta}} + \frac{2\log(8/\delta)}{3n \beta_{\min}} \;\Big|\; \{X_i\}_{i=1}^n \right) \leq \frac{\delta}{2} .
$$
\end{lemma}

\begin{proof}
{\it Step I: Symmetrization}.
Note that $\Is_\pi$ is fixed given $\{X_i\}_{i=1}^n$. For each $i\in\Is_\pi$, let $(A_i',Y_i')$ be an independent copy of $(A_i,Y_i)$ conditional on $X_i$, i.e., $(A_i',Y_i')\perp\!\!\!\perp (A_i,Y_i) \mid X_i$ and $(A_i',Y_i') \mid X_i \stackrel{d}{=} (A_i,Y_i) \mid X_i$. We have that for any $t\in\Rb$ and any $i\in\Is_\pi$,
\begin{align*}
	G(t\mid X_i,\pi) = \Eb[\1 \{Y_i\leq t\} \mid X_i, \pi(X_i)] = \Eb [ w_\pi(X_i,A_i) \1 \{Y_i\leq t\} \mid X_i ] = \Eb [ w_\pi(X_i,A_i') \1 \{Y_i'\leq t\} \mid X_i ].
\end{align*}
Therefore, for any $\lambda >0$, it holds that
\begin{align}\label{pf-eq-a1}
	&\Eb \left[ \exp\left( \lambda \left\| \frac{1}{n}\sum_{i\in\Is_\pi}w_\pi(X_i,A_i)\1\{Y_i\leq \cdot\} - G(\cdot\mid X_i,\pi) \right\|_\infty \right) \;\Big|\; \{X_i\}_{i=1}^n \right] \notag\\
	&= \Eb \left[ \exp\left( \lambda \left\| \frac{1}{n}\sum_{i\in\Is_\pi}w_\pi(X_i,A_i)\1\{Y_i\leq \cdot\} - \Eb [ w_\pi(X_i,A_i') \1 \{Y_i'\leq \cdot\} \mid X_i ] \right\|_\infty \right) \;\Big|\; \{X_i\}_{i=1}^n \right] \notag\\
	&= \Eb \left[ \exp\left( \lambda \left\| \Eb \left[\frac{1}{n}\sum_{i\in\Is_\pi}w_\pi(X_i,A_i)\1\{Y_i\leq \cdot\} -  w_\pi(X_i,A_i') \1 \{Y_i'\leq \cdot\} \;\Big|\; \{X_i,A_i,Y_i\}_{i=1}^n \right] \right\|_\infty \right) \;\Big|\; \{X_i\}_{i=1}^n \right]  \notag\\
	&\leq \Eb \left[ \exp\left( \lambda  \Eb \left[ \left\| \frac{1}{n}\sum_{i\in\Is_\pi}w_\pi(X_i,A_i)\1\{Y_i\leq \cdot\} -  w_\pi(X_i,A_i') \1 \{Y_i'\leq \cdot\} \right\|_\infty \;\Big|\; \{X_i,A_i,Y_i\}_{i=1}^n \right]  \right) \;\Big|\; \{X_i\}_{i=1}^n \right]  \notag\\
	&\leq \Eb \left[ \Eb \left[ \exp\left( \lambda   \left\| \frac{1}{n}\sum_{i\in\Is_\pi}w_\pi(X_i,A_i)\1\{Y_i\leq \cdot\} -  w_\pi(X_i,A_i') \1 \{Y_i'\leq \cdot\} \right\|_\infty   \right) \;\Big|\; \{X_i,A_i,Y_i\}_{i=1}^n \right] \;\Big|\; \{X_i\}_{i=1}^n \right]  \notag\\
	&= \Eb \left[ \exp\left( \lambda   \left\| \frac{1}{n}\sum_{i\in\Is_\pi}w_\pi(X_i,A_i)\1\{Y_i\leq \cdot\} -  w_\pi(X_i,A_i') \1 \{Y_i'\leq \cdot\} \right\|_\infty   \right)  \;\Big|\; \{X_i\}_{i=1}^n \right].
\end{align}
Here, Jensen's inequality and the tower property are successively used. Let $\{\varepsilon_i\}_{i\in\Is_\pi}$ be independent Rademacher variables, which are $1$ or $-1$ with probability $1/2$ each. It then follows from exchangeability that
\begin{align*}
	& \Cref{pf-eq-a1} = \Eb \left[ \exp\left( \lambda   \left\| \frac{1}{n}\sum_{i\in\Is_\pi} \varepsilon_i \left( w_\pi(X_i,A_i)\1\{Y_i\leq \cdot\} -  w_\pi(X_i,A_i') \1 \{Y_i'\leq \cdot\} \right) \right\|_\infty   \right)  \;\Big|\; \{X_i\}_{i=1}^n \right]. \\
	& \leq \Eb \left[ \exp\left( \lambda   \left\| \frac{1}{n}\sum_{i\in\Is_\pi} \varepsilon_i  w_\pi(X_i,A_i)\1\{Y_i\leq \cdot\} \right\|_\infty  +  \lambda \left\| \frac{1}{n}\sum_{i\in\Is_\pi} \varepsilon_i   w_\pi(X_i,A_i') \1 \{Y_i'\leq \cdot\}  \right\|_\infty \right)  \;\Big|\; \{X_i\}_{i=1}^n \right]. \\
	& \leq \Eb \left[ \frac{1}{2} \exp\left( 2\lambda   \left\| \frac{1}{n}\sum_{i\in\Is_\pi} \varepsilon_i  w_\pi(X_i,A_i)\1\{Y_i\leq \cdot\} \right\|_\infty  \right)  \right.\\
	& \qquad\qquad\qquad \left. + \frac{1}{2} \exp\left( 2\lambda   \left\| \frac{1}{n}\sum_{i\in\Is_\pi} \varepsilon_i  w_\pi(X_i,A_i')\1\{Y_i'\leq \cdot\} \right\|_\infty  \right)  \;\Big|\; \{X_i\}_{i=1}^n \right].  \\
	& =  \Eb \left[ \exp\left( 2\lambda   \left\| \frac{1}{n}\sum_{i\in\Is_\pi} \varepsilon_i  w_\pi(X_i,A_i)\1\{Y_i\leq \cdot\} \right\|_\infty  \right)  \;\Big|\; \{X_i\}_{i=1}^n  \right]\\
	& =\Eb \left[ \Eb_\varepsilon \exp\left(  \frac{2\lambda}{n}  \left\| \sum_{i\in\Is_\pi} \varepsilon_i  w_\pi(X_i,A_i)\1\{Y_i\leq \cdot\} \right\|_\infty  \right) \;\Big|\; \{X_i\}_{i=1}^n \right],
\end{align*}
where $\Eb_\varepsilon$ denotes the expectation with respect to the Rademacher variables only. 

{\it Step II: Bounding the moment generating function}. %exponential moment}.
For any realization $\{x_i,a_i,y_i\}_{i\in\Is_\pi}$ of $\{X_i,A_i,Y_i\}_{i\in\Is_\pi}$, without loss of generality, we relabel the indices so that $y_1\leq \ldots y_i\ldots\leq y_{|\Is_\pi|}$. Then we know that
$$
\left| \sum_{i=1}^{|\Is_\pi|} \varepsilon_i  w_\pi(x_i,a_i)\1\{y_i\leq t\} \right| = \left| \sum_{i=1}^{j} \varepsilon_i  w_\pi(x_i,a_i) \right|, \text{ with } j =\max \{i: y_i \leq t\}.
$$
Accordingly, the supremum norm can be expressed as
$$
\sup_{t\in\Rb} \left| \sum_{i=1}^{|\Is_\pi|} \varepsilon_i  w_\pi(x_i,a_i)\1\{y_i\leq t\} \right| = \max_{j\in[|\Is_\pi|]}\left| \sum_{i=1}^{j} \varepsilon_i  w_\pi(x_i,a_i) \right|.
$$
Thus,
\begin{align*}
&   \exp\left( \frac{2\lambda}{n} \left\| \sum_{i=1}^{|\Is_\pi|} \varepsilon_i  w_\pi(x_i,a_i)\1\{y_i\leq \cdot\} \right\|_\infty  \right) \\
& = \exp\left( \frac{2\lambda}{n} \max_{j}\left| \sum_{i=1}^{j} \varepsilon_i  w_\pi(x_i,a_i) \right|  \right) \\
& = \max_j \left\{  \exp\left(   \frac{2\lambda}{n}\sum_{i=1}^{j} \varepsilon_i  w_\pi(x_i,a_i) \right) \1\{ \sum_{i=1}^{j} \varepsilon_i  w_\pi(x_i,a_i)\geq 0 \} \right.\\
& \qquad\qquad\qquad \left. +  \exp\left(  -\frac{2\lambda}{n}\sum_{i=1}^{j} \varepsilon_i  w_\pi(x_i,a_i) \right) \1\{ \sum_{i=1}^{j} \varepsilon_i  w_\pi(x_i,a_i)< 0 \}  \right\}\\
& \leq \max_j \left\{  \exp\left(   \frac{2\lambda}{n}\sum_{i=1}^{j} \varepsilon_i  w_\pi(x_i,a_i) \right) \1\{ \sum_{i=1}^{j} \varepsilon_i  w_\pi(x_i,a_i)\geq 0 \} \right\} \\
& \qquad\qquad\qquad + \max_j \left\{  \exp\left(  -\frac{2\lambda}{n}\sum_{i=1}^{j} \varepsilon_i  w_\pi(x_i,a_i) \right) \1\{ \sum_{i=1}^{j} \varepsilon_i  w_\pi(x_i,a_i)< 0 \}  \right\}.
\end{align*}
By the symmetry of the Rademacher variables, together with \Cref{lem-maxIneq}, we obtain that
\begin{align}\label{pf-eq-a2}
& \Eb \exp\left( \frac{2\lambda}{n} \left\| \sum_{i=1}^{|\Is_\pi|} \varepsilon_i  w_\pi(x_i,a_i)\1\{y_i\leq \cdot\} \right\|_\infty  \right) \notag \\
& \leq 2 \Eb \max_j \left\{  \exp\left(   \frac{2\lambda}{n}\sum_{i=1}^{j} \varepsilon_i  w_\pi(x_i,a_i) \right) \1\{ \sum_{i=1}^{j} \varepsilon_i  w_\pi(x_i,a_i)\geq 0 \} \right\} \notag \\ 
& = 2 \Eb  \exp\left( \max_j  \frac{2\lambda}{n}\sum_{i=1}^{j} \varepsilon_i  w_\pi(x_i,a_i) \right) \1\{ \max_j \frac{2\lambda}{n}\sum_{i=1}^{j} \varepsilon_i  w_\pi(x_i,a_i)\geq 0 \} \notag \\
& \leq 4 \Eb  \exp\left(  \frac{2\lambda}{n}\sum_{i=1}^{|\Is_\pi|} \varepsilon_i  w_\pi(x_i,a_i) \right) \1\{ \frac{2\lambda}{n}\sum_{i=1}^{|\Is_\pi|} \varepsilon_i  w_\pi(x_i,a_i)\geq 0 \} \notag \\
& \leq 4 \Eb  \exp\left(  \frac{2\lambda}{n}\sum_{i=1}^{|\Is_\pi|} \varepsilon_i  w_\pi(x_i,a_i) \right) \notag \\
& = 4 \Pi_{i\in\Is_\pi}\Eb\exp \left(  \frac{2\lambda w_\pi(x_i,a_i) }{n}\varepsilon_i \right) \notag \\
& \leq 4 \exp \left( \sum_{i\in\Is_\pi} \frac{2\lambda^2w_\pi(x_i,a_i)^2}{n^2} \right),
\end{align}
where the last inequality follows from $\Eb e^{\lambda\varepsilon} = (e^{\lambda}+e^{-\lambda})/2 \leq e^{\lambda^2/2}$. After marginalization, we thus have
$$
\Cref{pf-eq-a1} \leq 4 \Eb\left[ \exp \left( \sum_{i\in\Is_\pi} \frac{2\lambda^2w_\pi(X_i,A_i)^2}{n^2} \right) \;\Big|\; \{X_i\}_{i=1}^n \right] \leq 4 \exp \left( \frac{2\lambda^2}{n^2} \sum_{i\in\Is_\pi} \frac{1}{\beta(X_i, \pi)^2} \right).
$$

Finally, Chernoff method yields that for any $\epsilon > 0$,
\begin{align*}
& \Pb\left( \left\| \frac{1}{n}\sum_{i\in\Is_\pi}w_\pi(X_i,A_i)\1\{Y_i\leq \cdot\} - G(\cdot\mid X_i,\pi) \right\|_\infty \geq \epsilon \;\Big|\; \{X_i\}_{i=1}^n \right) \\
& \leq \inf_{\lambda>0} 4\exp\left(  \frac{2\lambda^2}{n^2} \sum_{i\in\Is_\pi} \frac{1}{\beta(X_i, \pi)^2}  -\lambda\epsilon \right)\\
& = 4\exp\left(-\frac{n\epsilon^2}{8 \sigma_\pi^2 } \right).
\end{align*}
Taking $\epsilon = \sigma_\pi \sqrt{\frac{8}{n}\log\frac{8}{\delta}} $ yields the first result.

We now drive the Bernstein-type result by re-bounding \Cref{pf-eq-a1}. Using the fact $e^x \leq 1+x+\frac{x^2}{2(1-|x|/3)}$ for $|x|\leq3$, it holds for $\lambda$ satisfying $\frac{2\lambda }{n}w_\pi(x_i,a_i) <3$ that
$$
\Eb \exp\left( \frac{2\lambda w_\pi(x_i,a_i)}{n} \varepsilon_i \right) = \frac{1}{2} e^{ \frac{2\lambda }{n}w_\pi(x_i,a_i) } + \frac{1}{2}e^{-\frac{2\lambda}{n} w_\pi(x_i,a_i)} \leq 1 + \frac{\lambda^2\frac{4}{n^2} w_\pi(x_i,a_i)^2 }{2(1-\frac{\lambda}{3} \frac{2w_\pi(x_i,a_i)}{n})}.
$$
Note that, given $\{X_i\}_{i=1}^n$, the weights $w_\pi(X_i,A_i)$ are conditional independent with $\Eb[ w_\pi(X_i,A_i)^2 \mid \{X_i\}_{i=1}^n]= \frac{1}{\beta(X_i,\pi)}$ and  $ w_\pi(X_i,A_i)  \leq \frac{1}{\beta(X_i,\pi)}$. Hence, for $0 < \lambda < \min_{i\in\Is_\pi} \frac{3}{2}n\beta(X_i,\pi)$, it follows that
\begin{align*}
\Cref{pf-eq-a1} & \leq \Eb\left[ 4 \Pi_{i\in\Is_\pi} \Eb_\varepsilon \exp \left(  \frac{2\lambda w_\pi(X_i,A_i) }{n}\varepsilon_i \right) \;\Big|\; \{X_i\}_{i=1}^n \right]\\
& \leq \Eb\left[ 4 \Pi_{i\in\Is_\pi} \left( 1 + \frac{\lambda^2\frac{4}{n^2} w_\pi(X_i,A_i)^2 }{2(1-\frac{\lambda}{3} \frac{2w_\pi(X_i,A_i)}{n})} \right) \;\Big|\; \{X_i\}_{i=1}^n \right]\\
& \leq 4 \Pi_{i\in\Is_\pi}   \left( 1 + \frac{\lambda^2\frac{4}{n^2\beta(X_i,A_i)} }{2(1-\frac{\lambda}{3} \frac{2}{n\beta(X_i,\pi)})} \right) \\
& \leq  4 \Pi_{i\in\Is_\pi}  \exp \left( \frac{\lambda^2 \frac{4}{n^2\beta(X_i,\pi)}}{2\left(1- \frac{\lambda}{3} \frac{2}{n\beta(X_i,\pi)} \right)}\right) \\
& \leq 4  \exp \left( \frac{\lambda^2 \sum_{i\in\Is_\pi}  \frac{4}{n^2\beta(X_i,\pi)}}{2\left(1- \frac{\lambda}{3} \frac{2}{n\beta_{\min}} \right)}\right).
\end{align*}
Therefore, for any $\epsilon>0$,
$$
\Pb\left( \left\| \frac{1}{n}\sum_{i\in\Is_\pi}w_\pi(X_i,A_i)\1\{Y_i\leq \cdot\} - G(\cdot\mid X_i,\pi) \right\|_\infty \geq \epsilon \;\Big|\; \{X_i\}_{i=1}^n \right)  \leq  4\exp\left( \frac{\lambda^2  \frac{4}{n}(\sigma_\pi')^2 }{2\left(1- \frac{\lambda}{3} \frac{2}{n\beta_{\min}} \right)} -\lambda\epsilon \right).
$$
Choosing $\lambda = \frac{\epsilon}{\frac{4}{n}(\sigma_\pi')^2 + \frac{\epsilon}{3} \frac{2}{n\beta_{\min}}}$ as in the classical Bernstein inequality, we obtain
$$
\Pb\left( \left\| \frac{1}{n}\sum_{i\in\Is_\pi}w_\pi(X_i,A_i)\1\{Y_i\leq \cdot\} - G(\cdot\mid X_i,\pi) \right\|_\infty \geq \epsilon \;\Big|\; \{X_i\}_{i=1}^n \right)  \leq  4\exp\left( \frac{- \epsilon^2}{2\left( \frac{4}{n}(\sigma_\pi')^2 + \frac{\epsilon}{3} \frac{2}{n\beta_{\min}} \right)} \right).
$$
The proof is complete.

\end{proof}

%%%%%%%%%%%%%%%%%%%%%%%%%%%%%%%%%%%%%%%%%%%%

\begin{lemma}\label{lem-ope-2} 
For any $\delta\in(0,1)$, it holds that
$$\Pb\left( \left\|\frac{1}{n}\sum_{i=1}^{n} G(\cdot\mid X_i,\pi) - F^\pi(\cdot) \right\|_\infty \geq \sqrt{\frac{8}{n}\log\frac{8}{\delta}} \right) \leq \frac{\delta}{2}.$$
\end{lemma}

\begin{proof}
We begin with the symmetrization similarly. Let $\{X_i'\}_{i=1}^n$ be independent copies of $\{X_i\}_{i=1}^n$ and $\{\varepsilon_i\}_{i=1}^n$ be independent Rademacher variables. For any $\lambda>0$, we have
\begin{align*}
& \Eb\left[ \exp \left( \lambda \left\|\frac{1}{n}\sum_{i=1}^{n} G(\cdot\mid X_i,\pi) - F^\pi(\cdot) \right\|_\infty \right) \right] \\
& =\Eb\left[ \exp \left( \lambda \left\| \Eb\left[ \frac{1}{n}\sum_{i=1}^{n} G(\cdot\mid X_i,\pi) -  G(\cdot\mid X_i',\pi) \;\Big|\; \{X_i\}_{i=1}^n \right] \right\|_\infty \right) \right] \\
& \leq \Eb\left[ \exp \left( \lambda  \Eb\left[ \left\| \frac{1}{n}\sum_{i=1}^{n} G(\cdot\mid X_i,\pi) -  G(\cdot\mid X_i',\pi) \right\|_\infty \;\Big|\; \{X_i\}_{i=1}^n \right] \right) \right] \\
& \leq \Eb\left[ \Eb\left[ \exp \left( \lambda  \left\| \frac{1}{n}\sum_{i=1}^{n} G(\cdot\mid X_i,\pi) -  G(\cdot\mid X_i',\pi) \right\|_\infty  \right) \;\Big|\; \{X_i\}_{i=1}^n \right] \right] \\
& = \Eb\left[ \exp \left( \lambda   \left\| \frac{1}{n}\sum_{i=1}^{n} G(\cdot\mid X_i,\pi) -  G(\cdot\mid X_i',\pi) \right\|_\infty \right) \right] \\
& = \Eb\left[ \exp \left( \lambda   \left\| \frac{1}{n}\sum_{i=1}^{n} \varepsilon_i( G(\cdot\mid X_i,\pi) -  G(\cdot\mid X_i',\pi)) \right\|_\infty \right) \right] \\
& \leq \Eb\left[ \exp \left( 2\lambda   \left\| \frac{1}{n}\sum_{i=1}^{n} \varepsilon_i G(\cdot\mid X_i,\pi) \right\|_\infty \right) \right] \\
& = \Eb\left[ \Eb_\varepsilon\left[ \exp \left( 2\lambda   \left\| \frac{1}{n}\sum_{i=1}^{n} \varepsilon_i G(\cdot\mid X_i,\pi) \right\|_\infty \right) \right] \right].
\end{align*}

Next, we approximate the CDFs $G(\cdot\mid X_i,\pi)$ by step functions. Let $\{x_i\}_{i=1}^n$ be any realization  of $\{X_i\}_{i=1}^n$. For any $m\in\mathbb{N}_+$, we know from \Cref{lem-stepFunc} that there exists $\{(s_1^i,\ldots,s_m^i)\}_{i=1}^n\in\Rb^{m\times n}$ depending on $\{x_i\}_{i=1}^n$ such that, for step functions defined as $G_{m,i}(t) = \frac{1}{m}\sum_{j=1}^m\1\{ s_j^i \leq t\} $, $i\in[n]$, it holds
$$
\| G_{m,i}(\cdot) - G(\cdot\mid x_i,\pi) \|_\infty \leq \frac{1}{2m}.
$$
Therefore,
\begin{align*}
& \Eb \exp \left( 2\lambda   \left\| \frac{1}{n}\sum_{i=1}^{n} \varepsilon_i G(\cdot\mid x_i,\pi) \right\|_\infty \right)  \\
& \leq \Eb \exp \left( 2\lambda   \left\| \frac{1}{n}\sum_{i=1}^{n} \varepsilon_i G_{m,i}(\cdot) \right\|_\infty + 2\lambda\left\| \frac{1}{n}\sum_{i=1}^{n} \varepsilon_i( G(\cdot\mid x_i,\pi) - G_{m,i}(\cdot) )\right\|_\infty \right)\\
& \leq e^{\frac{\lambda}{m}}\Eb \exp \left( 2\lambda   \left\| \frac{1}{n}\sum_{i=1}^{n} \varepsilon_i G_{m,i}(\cdot) \right\|_\infty \right) \\
%& = e^{\frac{\lambda}{m}}\Eb \exp \left( 2\lambda   \left\|\frac{1}{m} \frac{1}{n} \sum_{j=1}^m \sum_{i=1}^{n} \varepsilon_i \1\{ s_j^i \leq t\} \right\|_\infty \right) \\
& \leq e^{\frac{\lambda}{m}}\Eb \exp \left( \frac{1}{m} \sum_{j=1}^m \frac{2\lambda}{n}  \left\|  \sum_{i=1}^{n} \varepsilon_i \1\{ s_j^i \leq \cdot\} \right\|_\infty \right) \\
& \leq e^{\frac{\lambda}{m}} \frac{1}{m} \sum_{j=1}^m  \Eb \exp \left( \frac{2\lambda}{n}  \left\|  \sum_{i=1}^{n} \varepsilon_i \1\{ s_j^i \leq \cdot\} \right\|_\infty \right).
\end{align*}
For each $j\in[m]$, analogously to \Cref{pf-eq-a2}, we can obtain the following bound for the exponential moment:
$$
\Eb \exp \left( \frac{2\lambda}{n}  \left\|  \sum_{i=1}^{n} \varepsilon_i \1\{ s_j^i \leq \cdot\} \right\|_\infty \right) \leq 4 \exp\left( \frac{2\lambda^2}{n} \right).
$$
Thus, 
$$
\Eb \exp \left( 2\lambda   \left\| \frac{1}{n}\sum_{i=1}^{n} \varepsilon_i G(\cdot\mid x_i,\pi) \right\|_\infty \right)  \leq e^{\frac{\lambda}{m}} 4 \exp\left( \frac{2\lambda^2}{n} \right).
$$
Since the above inequality holds for any $m\in\mathbb{N}_+$, taking $m\to\infty$ yields
\begin{equation}\label{pf-eq-a3}
\Eb \exp \left( 2\lambda   \left\| \frac{1}{n}\sum_{i=1}^{n} \varepsilon_i G(\cdot\mid x_i,\pi) \right\|_\infty \right)  \leq  4 \exp\left( \frac{2\lambda^2}{n} \right).
\end{equation}
Marginalizing it over $\{X_i\}_{i=1}^n$ then gives
$$
\Eb\left[ \exp \left( \lambda \left\|\frac{1}{n}\sum_{i=1}^{n} G(\cdot\mid X_i,\pi) - F^\pi(\cdot) \right\|_\infty \right) \right] \leq  4 \exp\left( \frac{2\lambda^2}{n} \right).
$$
Finally we have that, for any $\epsilon\in\Rb$,
\begin{align*}
\Pb\left( \left\|\frac{1}{n}\sum_{i=1}^{n} G(\cdot\mid X_i,\pi) - F^\pi(\cdot) \right\|_\infty \geq \epsilon \right) 
\leq \inf_{\lambda>0} 4\exp\left( \frac{2\lambda^2}{n}  -\lambda\epsilon \right) = 4\exp\left( -\frac{n\epsilon^2}{8} \right).
\end{align*}
Taking $\epsilon = \sqrt{\frac{8}{n}\log\frac{8}{\delta}}$ completes the proof.
\end{proof}
%%%%%%%%%%%%%%%%%%%%%%%%%%%%%%%%%%%%%%%%%%%%%%%%%%%%%%%%%%%%%%%%%%%%%%%%%%%%%%%%%%%%%%%%

%%%%%%%%%%%%%%%%%%%%%%%%%%%%%%%%%%%%%%%%%%%%%%%%%%%%%%%%%%%%%%%%%%%%%%%%%%%%%%%%%%%%%%%%
\subsection{Proof of \Cref{thm-uniBound-IS}}\label{pf-thm-uniBound-IS}
\begin{proof}
As before, we decompose the error as
\begin{align*}
\sup_{\pi\in\Pi} \left\| \Fis^\pi - F^\pi \right\|_\infty 
\leq \sup_{\pi\in\Pi} \left\| \frac{1}{n}\sum_{i=1}^{n}\hat{G}(\cdot\mid X_i,\pi) - G(\cdot\mid X_i,\pi) \right\|_\infty  + \sup_{\pi\in\Pi} \left\|\frac{1}{n}\sum_{i=1}^{n} G(\cdot\mid X_i,\pi) - F^\pi(\cdot) \right\|_\infty.
\end{align*}
For the first term, it holds by \Cref{pf-eq-a4} that, for any fixed $\pi\in\Pi$,
$$
\Pb\left( \left\| \frac{1}{n}\sum_{i=1}^{n}\hat{G}(\cdot\mid X_i,\pi) - G(\cdot\mid X_i,\pi) \right\|_\infty \geq    \sigma_\pi \sqrt{\frac{8}{n}\log\frac{8}{\delta}} + r_\pi \;\Big|\; \{X_i\}_{i=1}^n \right) \leq \frac{\delta}{2}.
$$

Recall the definitions of $ \hat{G}(\cdot\mid X_i,\pi), G(\cdot\mid X_i,\pi), \sigma_\pi$ and $ r_\pi$. Given $\{X_i\}_{i=1}^n$, we observe that different $\pi$ affect the above event only via $\pi(X_1),\ldots,\pi(X_n)$. Moreover, by \Cref{lem-Natarajan}, we know that 
$$|\{(\pi(X_1),\ldots,\pi(X_n)): \pi\in\Pi \}|\leq n^{d_{\Pi}} K^{2d_{\Pi}}.$$ 
Therefore, taking a union bound yields
\begin{align*}
\Pb\left( \sup_{\pi\in\Pi} \left\| \frac{1}{n}\sum_{i=1}^{n}\hat{G}(\cdot\mid X_i,\pi) - G(\cdot\mid X_i,\pi) \right\|_\infty -  \sigma_\pi \sqrt{\frac{8}{n}\log\frac{8}{\delta}} - r_\pi \geq 0 \;\Big|\; \{X_i\}_{i=1}^n \right) 
\leq \frac{\delta}{2}  n^{d_{\Pi}} K^{2d_{\Pi}}.
\end{align*}
After marginalization, it holds with probability at most $\delta/2$ that
$$
\sup_{\pi\in\Pi} \left\| \frac{1}{n}\sum_{i=1}^{n}\hat{G}(\cdot\mid X_i,\pi) - G(\cdot\mid X_i,\pi) \right\|_\infty - \sigma_\pi \sqrt{\frac{8}{n} \left( \log\frac{8}{\delta} +  d_{\Pi}  \log(nK^2)  \right)  } -  r_\pi \geq 0.
$$
Together with the bound of the second term in \Cref{lem-uniBound-1}, we finally have
\begin{align*}
\Pb \left( \sup_{\pi\in\Pi} \left\| \Fis^\pi - F^\pi \right\|_\infty -  (\sigma_\pi+2) \sqrt{ \frac{8}{n} \left( \log\frac{20}{\delta} + d_{\Pi}  \log(nK^2) \right) } - r_\pi \geq 0  \right) \leq \delta.
\end{align*}
The proof is complete.

\end{proof}

%%%%%%%%%%%%%%%%%%%%%%%%%%%%%%%%%%%%%%%%%%%%

\begin{lemma}\label{lem-uniBound-1}
For any $\delta\in(0,1)$, it holds that
$$
\Pb\left( \sup_{\pi\in\Pi} \left\|\frac{1}{n}\sum_{i=1}^{n} G(\cdot\mid X_i,\pi) - F^\pi(\cdot) \right\|_\infty \geq  2\sqrt{ \frac{8}{n} \left[ \log\frac{20}{\delta} +  d_{\Pi}  \log(nK^2) \right] } \right) \leq \frac{\delta}{2}.
$$
\end{lemma}

\begin{proof}
With loss of generality, we assume the supremum is attained, i.e., $ \exists\; \pi^\dagger\in\Pi$ s.t.
$$
\left\|\frac{1}{n}\sum_{i=1}^{n} G(\cdot\mid X_i,\pi^\dagger) - F^{\pi^\dagger}(\cdot) \right\|_\infty = \sup_{\pi\in\Pi} \left\|\frac{1}{n}\sum_{i=1}^{n} G(\cdot\mid X_i,\pi) - F^\pi(\cdot) \right\|_\infty .
$$
Otherwise, a limiting argument gives the same result. Note that $\pi^\dagger$ is fixed given $\{X_i\}_{i=1}^{n}$. By \Cref{lem-ope-2}, we know that for any $\epsilon \geq \sqrt{8\log20}$,
$$
\Pb\left( \left\|\frac{1}{n}\sum_{i=1}^{n} G(\cdot\mid X_i',\pi^\dagger) - F^{\pi^\dagger}(\cdot) \right\|_\infty \geq \frac{\epsilon}{\sqrt{n}}  \;\Big|\; \{X_i\}_{i=1}^{n} \right) \leq \frac{1}{5},
$$
where $\{X_i'\}_{i=1}^{n}$ are independent copies of $\{X_i\}_{i=1}^{n}$. Therefore, 
\begin{align*}
& \frac{4}{5} \Pb\left( \sup_{\pi\in\Pi} \left\|\frac{1}{n}\sum_{i=1}^{n} G(\cdot\mid X_i,\pi) - F^\pi(\cdot) \right\|_\infty \geq \frac{2\epsilon}{\sqrt{n}} \right) \\
& = \frac{4}{5} \Pb\left( \left\|\frac{1}{n}\sum_{i=1}^{n} G(\cdot\mid X_i,\pi^\dagger) - F^{\pi^\dagger}(\cdot) \right\|_\infty \geq \frac{2\epsilon}{\sqrt{n}} \right) \\
& \leq  \Eb\left[  \1\left\{ \left\|\frac{1}{n}\sum_{i=1}^{n} G(\cdot\mid X_i,\pi^\dagger) - F^{\pi^\dagger}(\cdot) \right\|_\infty \geq \frac{2\epsilon}{\sqrt{n}} \right\} \cdot \Pb\left( \left\|\frac{1}{n}\sum_{i=1}^{n} G(\cdot\mid X_i',\pi^\dagger) - F^{\pi^\dagger}(\cdot) \right\|_\infty \leq \frac{\epsilon}{\sqrt{n}}  \;\Big|\; \{X_i\}_{i=1}^{n} \right) \right] \\
& = \Pb\left( \left\|\frac{1}{n}\sum_{i=1}^{n} G(\cdot\mid X_i,\pi^\dagger) - F^{\pi^\dagger}(\cdot) \right\|_\infty \geq \frac{2\epsilon}{\sqrt{n}}  , \; \left\|\frac{1}{n}\sum_{i=1}^{n} G(\cdot\mid X_i',\pi^\dagger) - F^{\pi^\dagger}(\cdot) \right\|_\infty \leq \frac{\epsilon}{\sqrt{n}}  \right) \\
& \leq \Pb \left( \left\|\frac{1}{n}\sum_{i=1}^{n} G(\cdot\mid X_i,\pi^\dagger) - F^{\pi^\dagger}(\cdot) \right\|_\infty - \left\|\frac{1}{n}\sum_{i=1}^{n} G(\cdot\mid X_i',\pi^\dagger) - F^{\pi^\dagger}(\cdot) \right\|_\infty \geq \frac{\epsilon}{\sqrt{n}}  \right) \\
& \leq \Pb \left( \left\|\frac{1}{n}\sum_{i=1}^{n} G(\cdot\mid X_i,\pi^\dagger) - G(\cdot\mid X_i',\pi^\dagger) \right\|_\infty \geq \frac{\epsilon}{\sqrt{n}}  \right) \\
& \leq \Pb \left( \sup_{\pi\in\Pi} \left\|\frac{1}{n}\sum_{i=1}^{n} G(\cdot\mid X_i,\pi) - G(\cdot\mid X_i',\pi) \right\|_\infty \geq \frac{\epsilon}{\sqrt{n}}  \right) \\
& = \Pb \left( \sup_{\pi\in\Pi} \left\|\frac{1}{n}\sum_{i=1}^{n} \varepsilon_i( G(\cdot\mid X_i,\pi) - G(\cdot\mid X_i',\pi) ) \right\|_\infty \geq \frac{\epsilon}{\sqrt{n}}  \right) \\
& \leq 2 \Pb \left( \sup_{\pi\in\Pi} \left\|\frac{1}{n}\sum_{i=1}^{n} \varepsilon_i G(\cdot\mid X_i,\pi)  \right\|_\infty \geq \frac{\epsilon}{2\sqrt{n}}  \right).
\end{align*}
Above, the second equality holds since $\1\left\{ \left\|\frac{1}{n}\sum_{i=1}^{n} G(\cdot\mid X_i,\pi^\dagger) - F^{\pi^\dagger}(\cdot) \right\|_\infty \geq \frac{2\epsilon}{\sqrt{n}} \right\}$ is measurable with respect to $\{X_i\}_{i=1}^{n}$, the last equality follows from the exchangeability, and $\{\varepsilon_i\}_{i=1}^n$ are independent Rademacher variables. 

For any fixed $\pi\in\Pi$, using \Cref{pf-eq-a3} in the proof of \Cref{lem-ope-2}, we know that for any $\lambda>0$,
$$
\Eb \exp \left( \lambda   \left\| \frac{1}{n}\sum_{i=1}^{n} \varepsilon_i G(\cdot\mid X_i,\pi) \right\|_\infty \right)  \leq  4 \exp\left( \frac{\lambda^2}{2n} \right).
$$
Thus,
\begin{align*}
\Pb \left(  \left\|\frac{1}{n}\sum_{i=1}^{n} \varepsilon_i G(\cdot\mid X_i,\pi)  \right\|_\infty \geq \frac{\epsilon}{2\sqrt{n}}  \right) \leq \inf_{\lambda>0} 4\exp(\frac{\lambda^2}{2n} -\frac{\lambda \epsilon}{2\sqrt{n}}) = 4\exp(-\frac{\epsilon^2}{8}).
\end{align*}
Again, as different $\pi$ affect the above event only via $\pi(X_1),\ldots,\pi(X_n)$, \Cref{lem-Natarajan} and a union bound induce that
$$
\Pb \left( \sup_{\pi\in\Pi} \left\|\frac{1}{n}\sum_{i=1}^{n} \varepsilon_i G(\cdot\mid X_i,\pi)  \right\|_\infty \geq \frac{\epsilon}{2\sqrt{n}}  \right) \leq 4n^{d_{\Pi}} K^{2d_{\Pi}}\exp(-\frac{\epsilon^2}{8}).
$$

Finally, we obtain
$$
\Pb\left( \sup_{\pi\in\Pi} \left\|\frac{1}{n}\sum_{i=1}^{n} G(\cdot\mid X_i,\pi) - F^\pi(\cdot) \right\|_\infty \geq \frac{2\epsilon}{\sqrt{n}} \right) \leq 10 n^{d_{\Pi}} K^{2d_{\Pi}}\exp(-\frac{\epsilon^2}{8}).
$$
Taking $\epsilon = \sqrt{8\log\frac{20}{\delta} + 8 d_{\Pi}  \log(nK^2) } $ completes the proof.

\end{proof}

%%%%%%%%%%%%%%%%%%%%%%%%%%%%%%%%%%%%%%%%%%%%
%%%%%%%%%%%%%%%%%%%%%%%%%%%%%%%%%%%%%%%%%%%%%%%%%%%%%%%%%%%%%%%%%%%%%%%%%%%%%%%%%%%%%%%%

%%%%%%%%%%%%%%%%%%%%%%%%%%%%%%%%%%%%%%%%%%%%%%%%%%%%%%%%%%%%%%%%%%%%%%%%%%%%%%%%%%%%%%%%
\subsection{Proof of \Cref{thm-uniBound-WIS}}\label{pf-thm-uniBound-WIS}
\begin{proof}
It suffices to drive a conditional $1-\delta/2$ confidence upper bound, analogous to \Cref{lem-ope-1}, for 
$$ \left\| \frac{1}{n}\sum_{i\in\Is_\pi}\frac{1}{W_\pi} w_\pi(X_i,A_i)\1\{Y_i\leq \cdot\} - G(\cdot\mid X_i,\pi) \right\|_\infty, $$
after which the proof can be completed in the same manner as that of \Cref{thm-uniBound-IS}. Notice that
\begin{align*}
& \left\| \frac{1}{n}\sum_{i\in\Is_\pi}\frac{1}{W_\pi} w_\pi(X_i,A_i)\1\{Y_i\leq \cdot\} - G(\cdot\mid X_i,\pi) \right\|_\infty \\
& \leq  \left\| \frac{1}{n}\sum_{i\in\Is_\pi} \frac{1}{W_\pi}w_\pi(X_i,A_i)\1\{Y_i\leq \cdot\} - \frac{1}{W_\pi}G(\cdot\mid X_i,\pi) \right\|_\infty  
+  \left\| \frac{1}{n}\sum_{i\in\Is_\pi} \frac{1}{W_\pi}G(\cdot\mid X_i,\pi) - G(\cdot\mid X_i,\pi) \right\|_\infty \\
& \leq \frac{1}{W_\pi} \left\| \frac{1}{n}\sum_{i\in\Is_\pi} w_\pi(X_i,A_i)\1\{Y_i\leq \cdot\} - G(\cdot\mid X_i,\pi) \right\|_\infty + \frac{|1-W_\pi|}{W_\pi} \frac{|\Is_\pi|}{n}.
\end{align*}
Conditioning on $\{X_i\}_{i=1}^n$, as each $w_\pi(X_i,A_i)$ are independent and $\Eb[w_\pi(X_i,A_i)\mid \{X_i\}_{i=1}^n]=1$, Hoeffding's inequality ( \Cref{lem-Hoeffding}) implies
$$
\Pb \left( |W_\pi - 1 | \geq \tilde{\eta}_{\pi} \mid \{X_i\}_{i=1}^n \right) \leq \delta/4,
$$
where $\tilde{\eta}_{\pi} \coloneq \sqrt{ \frac{1}{2|\Is_\pi|} \log\frac{8}{\delta} } \sqrt{ \frac{1}{|\Is_\pi|} \sum_{i\in\Is_\pi} \frac{1}{\beta(X_i,\pi)^2}  }$. Note that $\tilde{\eta}_{\pi}$ is fixed given $\{X_i\}_{i=1}^n$. If $\tilde{\eta}_{\pi}<1$, we have $\frac{1}{W_\pi} < \frac{1}{1-\tilde{\eta}_{\pi}}$ on the event $\{|W_\pi - 1|< \tilde{\eta}_{\pi}\}$. Let $\tilde{\xi}_{\pi}\coloneq \frac{ \sigma_\pi }{ 1-\tilde{\eta}_{\pi} }\sqrt{\frac{8}{n}\log\frac{16}{\delta}} +  \frac{|\Is_\pi|}{n}\frac{\tilde{\eta}_{\pi}}{1-\tilde{\eta}_{\pi}}$. Applying \Cref{lem-ope-1} then gives
\begin{align*}
&	\Pb \left( \left\| \frac{1}{n}\sum_{i\in\Is_\pi}\frac{1}{W_\pi} w_\pi(X_i,A_i)\1\{Y_i\leq \cdot\} - G(\cdot\mid X_i,\pi) \right\|_\infty \geq \tilde{\xi}_{\pi} \;\Big|\; \{X_i\}_{i=1}^n \right) \\
& \leq \Pb \left(  \left\| \frac{1}{n}\sum_{i\in\Is_\pi} w_\pi(X_i,A_i)\1\{Y_i\leq \cdot\} - G(\cdot\mid X_i,\pi) \right\|_\infty \geq \sigma_\pi \sqrt{\frac{8}{n}\log\frac{16}{\delta}} \;\Big|\; \{X_i\}_{i=1}^n \right) \\ 
& \qquad\qquad\qquad + \Pb \left( |W_\pi - 1 | \geq \tilde{\eta}_{\pi} \mid \{X_i\}_{i=1}^n \right) \\
& \leq \delta/2.
\end{align*}
Thus, taking $\tilde{\alpha}_{\pi} \coloneq \1\{ \tilde{\eta}_{\pi} \geq 1 \} +  \1\{ \tilde{\eta}_{\pi} < 1 \}\cdot \tilde{\xi}_{\pi}$ yields a valid upper confidence bound for any fixed $\pi\in\Pi$:
$$
\Pb \left( \left\| \frac{1}{n}\sum_{i\in\Is_\pi}\frac{1}{W_\pi} w_\pi(X_i,A_i)\1\{Y_i\leq \cdot\} - G(\cdot\mid X_i,\pi) \right\|_\infty \leq \tilde{\alpha}_{\pi} \;\Big|\; \{X_i\}_{i=1}^n \right) \geq 1-\delta/2.
$$

The proof proceeds as in \Cref{thm-uniBound-IS}. Since different $\pi$ affect the above event only via $\pi(X_1),\ldots,\pi(X_n)$, \Cref{lem-Natarajan} and a union bound induce
\begin{align*}
\Pb \left( \sup_{\pi\in\Pi} \left\| \frac{1}{n}\sum_{i\in\Is_\pi}\frac{1}{W_\pi} w_\pi(X_i,A_i)\1\{Y_i\leq \cdot\} - G(\cdot\mid X_i,\pi) \right\|_\infty - \tilde{\alpha}_{\pi} > 0 \;\Big|\; \{X_i\}_{i=1}^n \right)  \leq n^{d_{\Pi}} K^{2d_{\Pi}}\delta/2.
\end{align*}
After adding the $r_\pi$ term, replacing $\delta$ by $\delta/n^{d_{\Pi}} K^{2d_{\Pi}}$ and marginalizing over $\{X_i\}_{i=1}^n$, it holds that
$$
\Pb \left( \sup_{\pi\in\Pi} \left\| \frac{1}{n}\sum_{i=1}^{n}\hat{H}(\cdot\mid X_i,\pi) - G(\cdot\mid X_i,\pi) \right\|_\infty - \alpha_{\pi} > 0 \right) \leq \frac{\delta}{2},
$$
where $\alpha_{\pi} \coloneq \1\{ \eta_{\pi} \geq 1 \} +  \1\{ \eta_{\pi} < 1 \} \cdot \xi_{\pi}'$, with $\xi_{\pi}' \coloneq \frac{ \sigma_\pi }{ 1-\eta_{\pi} }\sqrt{\frac{8}{n}[\log\frac{16}{\delta}+ d_{\Pi}  \log(nK^2) ]} +  \frac{|\Is_\pi|}{n}\frac{\eta_{\pi}}{1-\eta_{\pi}} + r_\pi$ and $\eta_{\pi} \coloneq \sqrt{ \frac{1}{2|\Is_\pi|} [\log\frac{8}{\delta} +  d_{\Pi}  \log(nK^2)] } \cdot \sqrt{ \frac{1}{|\Is_\pi|} \sum_{i\in\Is_\pi} \frac{1}{\beta(X_i,\pi)^2} }$

Finally, using a similar error decomposition and \Cref{lem-uniBound-1}, the proof is complete.

\end{proof}

%%%%%%%%%%%%%%%%%%%%%%%%%%%%%%%%%%%%%%%%%%%%%%%%%%%%%%%%%%%%%%%%%%%%%%%%%%%%%%%%%%%%%%%%

%%%%%%%%%%%%%%%%%%%%%%%%%%%%%%%%%%%%%%%%%%%%%%%%%%%%%%%%%%%%%%%%%%%%%%%%%%%%%%%%%%%%%%%%
\subsection{Proof of \Cref{thm-uniBound-DR}}\label{pf-thm-uniBound-DR}
\begin{proof}
We similarly make the following decomposition:
\begin{align*}
& \sup_{\pi\in\Pi} \left\| \Fdr^\pi - F^\pi \right\|_\infty \\
& \leq \sup_{\pi\in\Pi} \left\| \frac{1}{n}\sum_{i=1}^{n}\hat{\Gamma}(\cdot\mid X_i,\pi) - G(\cdot\mid X_i,\pi) \right\|_\infty  + \sup_{\pi\in\Pi} \left\|\frac{1}{n}\sum_{i=1}^{n} G(\cdot\mid X_i,\pi) - F^\pi(\cdot) \right\|_\infty.
\end{align*}
Note that the second term has already been bounded in \Cref{lem-uniBound-1}. For the first term, we know that for any fixed $\pi\in\Pi$,
\begin{align*}
& \left\| \frac{1}{n}\sum_{i=1}^{n}\hat{\Gamma}(\cdot\mid X_i,\pi) - G(\cdot\mid X_i,\pi) \right\|_\infty \\
& \leq \left\| \frac{1}{n}\sum_{i\in\Is_\pi}\hat{\Gamma}(\cdot\mid X_i,\pi) - G(\cdot\mid X_i,\pi) \right\|_\infty + \left\| \frac{1}{n}\sum_{i\notin\Is_\pi} \hat{\Gamma}(\cdot\mid X_i,\pi) - G(\cdot\mid X_i,\pi) \right\|_\infty \\
& = \left\| \frac{1}{n}\sum_{i\in\Is_\pi} w_\pi(X_i,A_i) \Big( \1 \{Y_i \leq \cdot\} - \bar{G}(\cdot\mid X_i,A_i) \Big) + \bar{G}(\cdot\mid X_i,\pi) - G(\cdot\mid X_i,\pi) \right\|_\infty \\
& \qquad\qquad\qquad + \left\| \frac{1}{n}\sum_{i\notin\Is_\pi} \bar{G}(\cdot\mid X_i,\pi) - G(\cdot\mid X_i,\pi) \right\|_\infty \\
& \leq \left\| \frac{1}{n}\sum_{i\in\Is_\pi} w_\pi(X_i,A_i)  \1 \{Y_i \leq \cdot\} - G(\cdot\mid X_i,\pi) \right\|_\infty + \left\| \frac{1}{n}\sum_{i\in\Is_\pi} w_\pi(X_i,A_i)  \bar{G}(\cdot\mid X_i,A_i) - \bar{G}(\cdot\mid X_i,\pi) \right\|_\infty + \bar{r}_\pi.
\end{align*}
By \Cref{lem-ope-1} and \Cref{lem-uniBound-DR-1}, which indicate that the first two terms above actually admit the same upper confidence bound, we obtain that for any fixed $\pi\in\Pi$,
\begin{align*}
\Pb\left(  \left\| \frac{1}{n}\sum_{i=1}^{n}\hat{\Gamma}(\cdot\mid X_i,\pi) - G(\cdot\mid X_i,\pi) \right\|_\infty  \geq  2\sigma_\pi \sqrt{\frac{8}{n}\log\frac{8}{\delta}} + \bar{r}_\pi \;\Biggl|\; \{X_i\}_{i=1}^n \right) \leq \delta.
\end{align*}
A similar union bound argument then gives
\begin{align*}
\Pb\left( \sup_{\pi\in\Pi} \left\| \frac{1}{n}\sum_{i=1}^{n}\hat{\Gamma}(\cdot\mid X_i,\pi) - G(\cdot\mid X_i,\pi) \right\|_\infty  -  2\sigma_\pi \sqrt{\frac{8}{n}\log\frac{8}{\delta}} - \bar{r}_\pi \geq 0 \;\Biggl|\; \{X_i\}_{i=1}^n \right)  \leq \delta  n^{d_{\Pi}} K^{2d_{\Pi}}.
\end{align*}
Marginalizing over $\{X_i\}_{i=1}^n$, we thus have
\begin{align*}
\Pb \Biggl(  \sup_{\pi\in\Pi} \left\| \frac{1}{n}\sum_{i=1}^{n}\hat{\Gamma}(\cdot\mid X_i,\pi) - G(\cdot\mid X_i,\pi) \right\|_\infty  - \bar{r}_\pi  -  2\sigma_\pi \sqrt{\frac{8}{n} [ \log\frac{16}{\delta}+ d_{\Pi}  \log(nK^2) ] }  \geq 0 \Biggl) \leq \frac{\delta}{2}.
\end{align*}

Finally, together with \Cref{lem-uniBound-1}, it holds with probability at least $1-\delta$ that
$$
\sup_{\pi\in\Pi} \left\| \Fdr^\pi - F^\pi \right\|_\infty -  2(\sigma_\pi + 1) \sqrt{\frac{8}{n} [ \log\frac{20}{\delta}+ d_{\Pi}  \log(nK^2) ] } - \bar{r}_\pi \leq 0.
$$
As \Cref{lem-monotone-trans} implies that $\left\| \Fdrc^\pi - F^\pi \right\|_\infty  \leq \left\| \Fdr^\pi - F^\pi \right\|_\infty$ always holds, the proof is complete.

\end{proof}

%%%%%%%%%%%%%%%%%%%%%%%%%%%%%%%%%%%%%%%%%%%%

\begin{lemma}\label{lem-uniBound-DR-1}
For any fixed $\pi\in\Pi$ and $\delta\in(0,1)$, it holds that
$$
\Pb\left( \left\| \frac{1}{n}\sum_{i\in\Is_\pi}w_\pi(X_i,A_i) \bar{G}(\cdot\mid X_i,A_i) - \bar{G}(\cdot\mid X_i,\pi) \right\|_\infty \geq \sigma_\pi \sqrt{\frac{8}{n}\log\frac{8}{\delta}}  \;\Big|\; \{X_i\}_{i=1}^n \right) \leq \frac{\delta}{2}.
$$
\end{lemma}

\begin{proof}
Note that $\bar{G}(t\mid X_i,\pi) = \Eb[ w(X_i, A_i'; \pi) \bar{G}(t\mid X_i,A_i') \mid \{X_i, A_i\}_{i=1}^n ]$ for each $i\in\Is_\pi$. For any $\lambda>0$, analogous to the proof of \Cref{lem-ope-1}, we have the following symmetrization.
\begin{align*}
&\Eb \left[ \exp\left( \lambda \left\| \frac{1}{n}\sum_{i\in\Is_\pi}w_\pi(X_i,A_i) \bar{G}(\cdot\mid X_i,A_i) - \bar{G}(\cdot\mid X_i,\pi) \right\|_\infty \right) \;\Big|\; \{X_i\}_{i=1}^n \right] \notag\\
&= \Eb \left[ \exp\left( \lambda \left\| \Eb \left[\frac{1}{n}\sum_{i\in\Is_\pi} w_\pi(X_i,A_i) \bar{G}(\cdot\mid X_i,A_i) -  w_\pi(X_i,A_i') \bar{G}(\cdot\mid X_i,A_i') \;\Big|\; \{X_i,A_i\}_{i=1}^n \right] \right\|_\infty \right) \;\Big|\; \{X_i\}_{i=1}^n \right]  \notag\\
&\leq \Eb \left[ \exp\left( \lambda  \Eb \left[ \left\| \frac{1}{n}\sum_{i\in\Is_\pi} w_\pi(X_i,A_i) \bar{G}(\cdot\mid X_i,A_i) -  w_\pi(X_i,A_i') \bar{G}(\cdot\mid X_i,A_i') \right\|_\infty \;\Big|\; \{X_i,A_i\}_{i=1}^n \right]  \right) \;\Big|\; \{X_i\}_{i=1}^n \right]  \notag\\
&\leq \Eb \left[ \exp\left( \lambda   \left\| \frac{1}{n}\sum_{i\in\Is_\pi} w_\pi(X_i,A_i) \bar{G}(\cdot\mid X_i,A_i) -  w_\pi(X_i,A_i') \bar{G}(\cdot\mid X_i,A_i') \right\|_\infty   \right)  \;\Big|\; \{X_i\}_{i=1}^n \right]. \\
&= \Eb \left[ \exp\left( \lambda   \left\| \frac{1}{n}\sum_{i\in\Is_\pi} \varepsilon_i ( w_\pi(X_i,A_i) \bar{G}(\cdot\mid X_i,A_i) -  w_\pi(X_i,A_i') \bar{G}(\cdot\mid X_i,A_i') ) \right\|_\infty   \right)  \;\Big|\; \{X_i\}_{i=1}^n \right]. \\
&\leq \Eb \left[ \exp\left( 2\lambda   \left\| \frac{1}{n}\sum_{i\in\Is_\pi} \varepsilon_i w_\pi(X_i,A_i) \bar{G}(\cdot\mid X_i,A_i)  \right\|_\infty   \right)  \;\Big|\; \{X_i\}_{i=1}^n \right].
\end{align*}

We next approximate $ \bar{G}(\cdot\mid X_i,A_i) $ by step functions as we did in \Cref{lem-ope-2}. Conditional on any realization $\{x_i, a_i\}_{i=1}^n$ of $\{X_i, A_i\}_{i=1}^n$, we know from \Cref{lem-stepFunc} that there exists $\{(s_1^i,\ldots,s_m^i)\}_{i\in\Is_\pi}$ and $ \bar{G}_{m,i}(t) = \frac{1}{m}\sum_{j=1}^m\1\{ s_j^i \leq t\} $ such that
$
\| \bar{G}_{m,i}(\cdot) -  \bar{G}(\cdot\mid x_i,a_i)  \|_\infty \leq \frac{1}{2m}. 
$
Then
\begin{align*}
& \Eb_\varepsilon  \exp\left( 2\lambda   \left\| \frac{1}{n}\sum_{i\in\Is_\pi} \varepsilon_i w_\pi(x_i,a_i) \bar{G}(\cdot\mid x_i,a_i)  \right\|_\infty   \right)  \\
%& \leq \Eb_\varepsilon \exp \left( 2\lambda   \left\| \frac{1}{n}\sum_{i=1}^{n} \varepsilon_i w_\pi(x_i,a_i) \bar{G}^m(t; x_i, \pi) \right\|_\infty + 2\lambda\left\| \frac{1}{n}\sum_{i=1}^{n} \varepsilon_i  w_\pi(x_i,a_i) ( \bar{G}^m(t,x_i,\pi) -  \bar{G}(t;x_i,\pi) )\right\|_\infty \right)\\
& \leq \exp(\frac{\lambda}{m}\max_{i\in\Is_\pi}  w_\pi(x_i,a_i) ) \Eb_\varepsilon \exp \left( 2\lambda   \left\| \frac{1}{n}\sum_{i\in\Is_\pi} \varepsilon_i w_\pi(x_i,a_i) \bar{G}_{m,i}(\cdot) \right\|_\infty \right) \\
& = \exp(\frac{\lambda}{m}\max_{i\in\Is_\pi}  w_\pi(x_i,a_i) )  \Eb_\varepsilon \exp \left( 2\lambda   \left\|\frac{1}{m} \frac{1}{n} \sum_{j=1}^m \sum_{i\in\Is_\pi} \varepsilon_i w_\pi(x_i,a_i) \1\{ s_j^i \leq \cdot\} \right\|_\infty \right) \\
%& \leq \exp(\frac{\lambda}{m}\max_{i\in\Is_\pi}  w_\pi(x_i,a_i) ) \Eb_\varepsilon \exp \left( \frac{1}{m} \sum_{j=1}^m \frac{2\lambda}{n}  \left\|  \sum_{i\in\Is_\pi} \varepsilon_i w_\pi(x_i,a_i) \1\{ s_j^i \leq t\} \right\|_\infty \right) \\
& \leq \exp(\frac{\lambda}{m}\max_{i\in\Is_\pi}  w_\pi(x_i,a_i) ) \frac{1}{m} \sum_{j=1}^m \Eb_\varepsilon \exp \left( \frac{2\lambda}{n}  \left\|  \sum_{i\in\Is_\pi} \varepsilon_i w_\pi(x_i,a_i) \1\{ s_j^i \leq \cdot\} \right\|_\infty \right).
\end{align*}
For each $j\in[m]$, analogously to \Cref{pf-eq-a2}, we can obtain that
$$
\Eb_\varepsilon \exp \left( \frac{2\lambda}{n}  \left\|  \sum_{i\in\Is_\pi} \varepsilon_i w_\pi(x_i,a_i) \1\{ s_j^i \leq t\} \right\|_\infty \right) \leq 4 \exp\left( \frac{2\lambda^2}{n^2}\sum_{i\in\Is_\pi}w_\pi(x_i,a_i)^2 \right).
$$
Taking $m\to\infty$ and marginalizing over $\{A_i\}_{i=1}^n$ yield
$$
\Eb \left[ \exp\left( \lambda \left\| \frac{1}{n}\sum_{i\in\Is_\pi}w_\pi(X_i,A_i) \bar{G}(\cdot\mid X_i,A_i) - \bar{G}(\cdot\mid X_i,\pi) \right\|_\infty \right) \;\Big|\; \{X_i\}_{i=1}^n \right] \leq 4 \exp\left( \frac{2\lambda^2}{n}\sigma_\pi^2 \right).
$$
A final Chernoff argument completes the proof.

\end{proof}

%%%%%%%%%%%%%%%%%%%%%%%%%%%%%%%%%%%%%%%%%%%%%%%%%%%%%%%%%%%%%%%%%%%%%%%%%%%%%%%%%%%%%%%%

%%%%%%%%%%%%%%%%%%%%%%%%%%%%%%%%%%%%%%%%%%%%%%%%%%%%%%%%%%%%%%%%%%%%%%%%%%%%%%%%%%%%%%%%
\subsection{Proof of \Cref{cor-bound-dr}} \label{pf-cor-bound-dr}
\begin{proof}
Note that, by using the Bernstein-style inequality in \Cref{lem-ope-1}, one can similarly establish the conclusion of \Cref{thm-uniBound-DR} with
$$
R(\pi)\coloneq 2(\sigma_\pi'+1) \sqrt{\frac{8}{n} [ \log\frac{20}{\delta} +  d_{\Pi} \log(nK^2) ] } + \frac{2}{3n\beta_{\min}}[ \log\frac{16}{\delta} +  d_{\Pi} \log(nK^2) ] + \bar{r}_\pi.
·$$
Under the assumption that $\beta(x,\pi^*)\geq \beta_{\inf}$ for $\Pb_X$-almost all $x$, we have $\sigma_{\pi^*}' = \sqrt{\frac{1}{n}\sum_{i\in[n]}\frac{1}{\beta(X_i, \pi^*)}} \leq \frac{1}{\sqrt{\beta_{\inf}}}$ and $\bar{r}_{\pi^*}=0$. Therefore,
\begin{align*}
R(\pi^*) & \leq 2 ( \frac{1}{\sqrt{\beta_{\inf}}} + 1) \sqrt{\frac{8}{n} [ \log\frac{20}{\delta} +  d_{\Pi} \log(nK^2) ] } + \frac{2}{3n\beta_{\inf}}[ \log\frac{16}{\delta} +  d_{\Pi} \log(nK^2) ] \\
& \leq (8\sqrt{2} + \frac{2}{3} \sqrt{ \frac{\log\frac{20}{\delta} +  d_{\Pi} \log(nK^2)}{n\beta_{\inf}} } ) \sqrt{ \frac{\log\frac{20}{\delta} +  d_{\Pi} \log(nK^2)}{n\beta_{\inf}} } \\
& \leq (16\sqrt{2} + \frac{4}{3}\sqrt{c_0}) \sqrt{ \frac{\log\frac{20}{\delta} \cdot d_{\Pi} \log(nK^2)}{n\beta_{\inf}} }.
\end{align*}
Applying \Cref{pro-pessi} completes the proof.
\end{proof}
%%%%%%%%%%%%%%%%%%%%%%%%%%%%%%%%%%%%%%%%%%%%%%%%%%%%%%%%%%%%%%%%%%%%%%%%%%%%%%%%%%%%%%%%

%%%%%%%%%%%%%%%%%%%%%%%%%%%%%%%%%%%%%%%%%%%%%%%%%%%%%%%%%%%%%%%%%%%%%%%%%%%%%%%%%%%%%%%%
\subsection{Proof of \Cref{thm-minimax}}\label{pf-thm-minimax}

This proof builds on the argument of Theorem 1 in \citet{zhan2024policy}, albeit with several adaptations.% to our setting.

\begin{proof}
Let $S=\{s_1,\ldots,s_d\}\subset\Xs$ with $d=d_\Pi$ be the set shattered by $\Pi$. By \Cref{def-Natarajan-dim}, there exists two functions $f_1, f_{-1}:\Ss\mapsto[K]$ such that $f_1(s_i) \neq f_{-1}(s_i)$ for all $i\in[d]$, and moreover, for any vector $\theta=(\theta_1, \ldots,\theta_d)\in\Theta\coloneq \{\pm1\}^d$, there exists a policy $\pi_\theta \in\Pi$ satisfying $\pi_\theta(s_i) = f_{\theta_i}(s_i), \forall i\in[d]$.

We now construct a collection of $2^d$ instances in $\Vs$. For each $\theta\in\Theta$, define $v_\theta$ as the bandit environment whose contextual distribution $\Pb_X$ is uniform over $S$, and the conditional reward is given by
$$
Y|X=s_i,A=a \sim 
\begin{cases}
\Bern(1/2+\theta_i\delta),	& \text{if } a=f_1(s_i), \\
\Bern(1/2),		& \text{if } a=f_{-1}(s_i), \\
\Bern(0),  & \text{otherwise},
\end{cases}
$$ 
where $\delta\in(0,1/4)$ is a constant to be specified later, and $\Bern(\cdot)$ denotes the Bernoulli distribution. We next verify that $\pi_\theta$ is an optimal policy under the environment $v_\theta$. For any policy $\pi\in\Pi$, observe first that $F^{\pi_\theta,v_\theta}(t) = F^{\pi,v_\theta}(t) = 0$ or $1$ if $t<0$ or $t\geq1$. Fix $t\in[0,1)$. For indices $i$ satisfying $\theta_i=1$, we have
$$
\Pb(Y\leq t \mid s_i, \pi_\theta) = 1/2-\delta,\qquad \Pb(Y\leq t \mid s_i, \pi) =  \begin{cases}
1/2 - \delta,	& \text{if } \pi(s_i) = \pi_\theta(s_i) = f_1(s_i), \\
1/2,	& \text{if } \pi(s_i) = f_{-1}(s_i), \\
1.	& \text{otherwise}.
\end{cases}
$$
Similarly, for $i$ with $\theta_i=-1$,
$$
\Pb(Y\leq t \mid s_i, \pi_\theta) = 1/2,\qquad \Pb(Y\leq t \mid s_i, \pi) =  \begin{cases}
1/2,	& \text{if } \pi(s_i) = \pi_\theta(s_i) = f_{-1}(s_i), \\
1/2 + \delta,	& \text{if } \pi(s_i) = f_1(s_i), \\
1.	& \text{otherwise}.
\end{cases}
$$
Therefore,
\begin{equation}\label{pf-eq-a5}
F^{\pi,v_\theta}(t) - F^{\pi_\theta,v_\theta}(t) = \frac{1}{d}\sum_{i=1}^d [\Pb(Y\leq t \mid s_i, \pi) - \Pb(Y\leq t \mid s_i, \pi_\theta)] \geq \frac{1}{d}\sum_{i=1}^d \delta\cdot\1\{ \pi(s_i)\neq\pi_\theta(s_i) \} \geq 0.
\end{equation}
Consequently, $F^{\pi_\theta,v_\theta}(t) \leq F^{\pi,v_\theta}(t)$ holds for all $t\in\Rb$. By the monotonicity assumption on $\rho$, this implies $\rho(F^{\pi_\theta,v_\theta})\geq\rho(F^{\pi,v_\theta})$, thereby establishing the optimality of $\pi_\theta$ under $v_\theta$. In addition, define a fixed behavior policy $\beta$ by setting
$$
\beta(s_i,a)=
\begin{cases}
\beta_{\inf}, & \text{if } a=f_1(s_i) \text{ or } f_{-1}(s_i), \\
(1-2\beta_{\inf})/(K-2), &\text{otherwise}.
\end{cases}
$$ 
By construction, we know that $(v_\theta,\beta)\in\Vs$ for every $\theta\in\Theta$. 

Let $\Pb_{\theta}$ and $\Eb_{\theta}$ denote the probability measure and expectation with respect to the data distribution induced by $(v_\theta,\beta)$. Then, for any (data-dependent) policy $\hat{\pi}$, it follows from \Cref{pf-eq-a5} that	
\begin{align}\label{pf-eq-a6}
\sup_{(v,\beta)\in\Vs} \Eb\,W_1( F^{\pi^*_v,v} , F^{\hat{\pi},v} ) &\geq \max_{\theta\in\Theta} \Eb_\theta W_1(F^{\pi_\theta,v_\theta} , F^{\hat{\pi},v_\theta} ) \notag \\
& = \max_{\theta\in\Theta} \Eb_\theta \int_{0}^{1}| F^{\pi_\theta,v_\theta}(t) - F^{\hat{\pi},v_\theta}(t) |\drm t \notag \\
&\geq \frac{1}{2^d} \sum_{\theta\in\Theta} \frac{1}{d}\sum_{i=1}^d \delta \cdot \Eb_\theta \1\{ \hat{\pi}(s_i)\neq\pi_\theta(s_i) \},
\end{align}
where the equality follows from an equivalent characterization of the Wasserstein distance between two CDFs (see, e.g., \citealt{prashanth2022a}, Lemma 2). For any $\theta$ and $i\in[d]$, let $\theta^{-i}\coloneq(\theta_1,\ldots,-\theta_i,\ldots,\theta_d)$ denotes the vector with the $i$-th coordinate flipped. Then,
\begin{align*}
\Cref{pf-eq-a6} & \geq \frac{\delta}{2^d \cdot d} \sum_{i=1}^d \sum_{\theta\in\Theta} \Eb_\theta \1\{ \hat{\pi}(s_i)\neq\pi_\theta(s_i) \} \\
& = \frac{\delta}{2^d \cdot d} \sum_{i=1}^d \sum_{\theta\in\Theta:\theta_i=1} [ \Eb_\theta \1\{ \hat{\pi}(s_i)\neq\pi_\theta(s_i) \} + \Eb_{\theta^{-i}} \1\{ \hat{\pi}(s_i)\neq\pi_{\theta^{-i}}(s_i) \} ] \\
& = \frac{\delta}{2^d \cdot d} \sum_{i=1}^d \sum_{\theta\in\Theta:\theta_i=1} [ \Eb_\theta \1\{ \hat{\pi}(s_i)\neq f_1(s_i) \} + \Eb_{\theta^{-i}} \1\{ \hat{\pi}(s_i)\neq f_{-1}(s_i) \} ] \\
& \geq \frac{\delta}{2^d \cdot d} \sum_{i=1}^d \sum_{\theta\in\Theta:\theta_i=1} [ \Eb_\theta \1\{ \hat{\pi}(s_i)\neq f_1(s_i) \} + \Eb_{\theta^{-i}} \1\{ \hat{\pi}(s_i) = f_1(s_i) \} ] \\
& \geq  \frac{\delta}{2^d \cdot d} \sum_{i=1}^d \sum_{\theta\in\Theta:\theta_i=1} \frac{1}{2}\exp\left(-\KL(\Pb_{\theta},\Pb_{\theta^{-i}})\right),
\end{align*} 
where the last inequality follows from the Bretagnolle–Huber inequality (see, e.g., \citealt{lattimore2020bandit}, Theorem 14.2). Since $\theta$ and $\theta^{-i}$ differ only in the $i$-th coordinate, the likelihood ratio satisfies
\begin{align*}
\log \frac{\Pb_{\theta}}{\Pb_{\theta^{-i}}}(x_1,\ldots,y_n) &= \sum_{j=1}^n \log \frac{\Pb_{\theta}}{\Pb_{\theta^{-i}}} (y_j\mid x_j,a_j) \\
& = \sum_{j=1}^n \1\{x_j=s_i,a_j=f_1(s_i)\} [\1\{y_j=1\}\log\frac{1/2+\delta}{1/2-\delta} + \1\{y_j=0\}\log\frac{1/2-\delta}{1/2+\delta}].
\end{align*}
Taking expectations yields
\begin{align*}
\KL(\Pb_{\theta},\Pb_{\theta^{-i}}) & = \Eb_\theta \log \frac{\Pb_{\theta}}{\Pb_{\theta^{-i}}}(X_1,\ldots,Y_n) \\
& = \sum_{j=1}^n \frac{\beta_{\inf}}{d}[(1/2+\delta)\log\frac{1/2+\delta}{1/2-\delta} + (1/2-\delta)\log\frac{1/2-\delta}{1/2+\delta}] \\
& = \sum_{j=1}^n \frac{2\delta\beta_{\inf}}{d} [\log(1+2\delta) - \log(1-2\delta)]  \leq \frac{12n\delta^2\beta_{\inf}}{d},
\end{align*}
where the last inequality follows from $\log(1+x)\leq x$ and $\log(1-x)\geq-2x$ for $x\in(0,1/2)$. Substituting this bound back into \Cref{pf-eq-a6}, we obtain
$$
\Cref{pf-eq-a6} \geq \frac{\delta}{2^d \cdot d} \sum_{i=1}^d \sum_{\theta\in\Theta:\theta_i=1} \frac{1}{2} \exp(-12n\delta^2\beta_{\inf}/d) = \frac{\delta}{4} \exp(-12n\delta^2\beta_{\inf}/d).
$$
Choosing $\delta=\sqrt{d/(24n\beta_{\inf})}$ completes the proof.
\end{proof}

%%%%%%%%%%%%%%%%%%%%%%%%%%%%%%%%%%%%%%%%%%%%%%%%%%%%%%%%%%%%%%%%%%%%%%%%%%%%%%%%%%%%%%%%

%%%%%%%%%%%%%%%%%%%%%%%%%%%%%%%%%%%%%%%%%%%%%%%%%%%%%%%%%%%%%%%%%%%%%%%%%%%%%%%%%%%%%%%%
%%%%%%%%%%%%%%%%%%%%%%%%%%%%%%%%%%%%%%%%%%%%%%%%%%%%%%%%%%%%%%%%%%%%%%%%%%%%%%%%%%%%%%%%
\section{Auxiliary Lemmas}

\begin{lemma}\label{lem-maxIneq}
Let $\varepsilon_1,\ldots,\varepsilon_n$ be independent Rademacher variables with any $n\geq1$. For any $(w_1,\ldots,w_n)\in\Rb^n$, it holds
$$
\Eb  \exp\left( \max_{j\in[n]}  \sum_{i=1}^{j} w_i\varepsilon_i  \right) \1\{ \max_{j\in[n]}  \sum_{i=1}^{j} w_i\varepsilon_i \geq 0 \}  \leq  2 \Eb  \exp\left(  \sum_{i=1}^{n} w_i\varepsilon_i  \right) \1\{ \sum_{i=1}^{n} w_i\varepsilon_i  \geq 0 \}.
$$
\end{lemma}
\begin{proof}
We first prove that, for any $t\geq0$,
\begin{equation}\label{pf-eq-c1}
\Pb\left( \max_{j\in[n]}\sum_{i=1}^j w_i\varepsilon_i \geq t \right) \leq 2\Pb\left( \sum_{i=1}^n w_i\varepsilon_i \geq t \right).
\end{equation}
Define disjoint events $E_1 \coloneq \{ w_1\varepsilon_1\geq t \}$ and $E_j \coloneq \left\{ \sum_{i=1}^j w_i\varepsilon_i \geq t, \sum_{i=1}^i w_i\varepsilon_i < t, \forall i<j \right\}$ for $j=2,\ldots,n$. Clearly,
$$
\cup_j\Big(  E_j\cap \Big\{ \sum_{i>j} w_i\varepsilon_i \geq 0 \Big\} \Big) \subset \Big\{  \sum_{i=1}^n w_i\varepsilon_i \geq t \Big\}.
$$
Note that $\Pb\left( \sum_{i>j} w_i\varepsilon_i \geq 0 \right)\geq 1/2$, since $\sum_{i>j} w_i\varepsilon_i$ is centered and symmetric. By the independence, we have
\begin{align*}
\Pb\Big(  \sum_{i=1}^n w_i\varepsilon_i \geq t \Big) \geq& \Pb\Big( \cup_j\Big(  E_j\cap \Big\{ \sum_{i>j} w_i\varepsilon_i \geq 0 \Big\} \Big) \Big\} \Big)\\
=& \sum_{j} \Pb\Big(  E_j\bigcap \Big\{ \sum_{i>j} w_i\varepsilon_i \geq 0 \Big\} \Big)\\
=& \sum_{j}\Pb\left(E_j\right)\Pb\Big(\sum_{i>j} w_i\varepsilon_i \geq 0\Big)\\
\geq& \sum_{j}\frac12\Pb(E_j) =\frac12\Pb\left( \cup_jE_j \right).
\end{align*}
Noting that $	\left\{  \max_{j\in[n]}\sum_{i=1}^j w_i\varepsilon_i \geq t   \right\} \subset \cup_jE_j$, we thus obtain \Cref{pf-eq-c1}. Moreover, by Fubini's theorem, 
\begin{align*}
& \Eb [ \exp\left( \max_j  \sum_{i=1}^{j} w_i\varepsilon_i \right)  \1\{ \max_j  \sum_{i=1}^{j} w_i\varepsilon_i  \geq0\} ] \\
& = \Eb [ (\int_{0}^{\max_j  \sum_{i=1}^{j} w_i\varepsilon_i} e^t \drm t + e^0 )\1\{ \max_j  \sum_{i=1}^{j} w_i\varepsilon_i \geq0\} ]\\
& = \Eb [ \int_{0}^{+\infty} e^t\1\{t\leq \max_j  \sum_{i=1}^{j} w_i\varepsilon_i\} \drm t + \1\{\max_j  \sum_{i=1}^{j} w_i\varepsilon_i\geq0\}  ]\\
& = \Pb(\max_j  \sum_{i=1}^{j} w_i\varepsilon_i\geq0) + \int_{0}^{+\infty} e^t \Pb(\max_j  \sum_{i=1}^{j} w_i\varepsilon_i\geq t) \drm t.\\
& \leq 2 \Pb(  \sum_{i=1}^{n} w_i\varepsilon_i\geq0) + \int_{0}^{+\infty} e^t 2 \Pb(  \sum_{i=1}^{n} w_i\varepsilon_i\geq t) \drm t.\\
& = 2 \Eb [ \exp\left(  \sum_{i=1}^{n} w_i\varepsilon_i \right)  \1\{  \sum_{i=1}^{n} w_i\varepsilon_i  \geq0\} ].
\end{align*}
The proof is complete.
\end{proof}

\begin{lemma}\label{lem-Bernstein}
Let $X$ be a random variable satisfying $\Eb X=0$ and $|X| \leq M$ almost surely. Then, for any $0 <\lambda<3/M$, it holds
$$
\Eb e^{\lambda X} \leq \exp \left( \frac{\lambda^2 \Eb X^2}{2\left(1- \lambda M/3\right)}\right).
$$
\end{lemma}

\begin{lemma}\label{lem-stepFunc}
For any cumulative distribution function $F$ and any $m\in\mathbb{N}_+$, there exits $\{s_1,\ldots,s_m\} \in \Rb^m$ such that, for a step function of the form
$$
F_m(t) \coloneq \frac{1}{m} \sum_{j=1}^m \1\{s_j \leq t\}, \; t\in\Rb,
$$
it holds that
$$
\| F - F_m \|_\infty \leq \frac{1}{2m}.
$$
\end{lemma}
\begin{proof}
For any $j\in[m]$, let $s_j$ be the $\frac{2j-1}{2m}$-quantile of $F$, i.e., $$s_j \coloneq \inf\{t: F^\pi(t) \geq \frac{2j-1}{2m} \}.$$
With these choices of $s_j$, we next show $|F^\pi(t) - F_m(t)| \leq \frac{1}{2m}$ for any $ t\in\Rb$. 

For $t$ such that $F^\pi(t)\in[\frac{2k-1}{2m}, \frac{2k+1}{2m})$ with some $k\in[m-1]$, we have
$$
s_k = \inf\{t: F^\pi(t) \geq \frac{2k-1}{2m} \} \leq t \text{ \quad and \quad}
s_{k+1} = \inf\{t: F^\pi(t) \geq \frac{2k+1}{2m} \} > t,
$$
which implies $F_m(t) = \frac{1}{m} \sum_{j=1}^m \1\{s_j \leq t\} = \frac{k}{m}$. Thus, $ |F^\pi(t) - F_m(t)| \leq \frac{1}{2m}$. One can similarly verify the correctness for $t$ such that $F^\pi(t)\in[0, \frac{1}{2m})\cup [\frac{2m-1}{2m}, 1]$. The proof is complete.
\end{proof}

\begin{lemma}[\citealt{natarajan1989learning} ]\label{lem-Natarajan}
Let $\Ss, \As$ be two finite sets and $\Pi$ be a set of functions from $\Ss$ to $\As$. Then, $|\Pi|\leq |\Ss|^{\mathrm{Ndim}(\Pi)}|\As|^{2\mathrm{Ndim}(\Pi)}$, where $\mathrm{Ndim}$ is the Natarajan dimension of $\Pi$.
\end{lemma}

\begin{lemma}[Hoeffding's inequality]\label{lem-Hoeffding}
Let $ X_1,\ldots,X_n$ be independent random variables such that $a_i\leq X_i\leq b_i$ almost surely. Let $S_n\coloneq\sum_{i=1}^n X_i$. It holds for any $t>0$ that
$$
\Pb(|S_n - \Eb S_n| \geq t ) \leq 2\exp\left(-\frac{2t^2}{\sum_{i=1}^n (b_i-a_i)^2 }\right).
$$
\end{lemma}

\begin{lemma}\label{lem-monotone-trans}	
For any c.d.f. $F_0$ and any function $F:\Rb\mapsto\Rb$ with finite $\| F - F_0 \|_\infty$, define the monotone transformation of $F$ as
$
F_\mathrm{M}(t) \coloneq \sup_{t'\leq t} F(t').
$
It holds that 
$$
\| F_\mathrm{M} - F_0 \|_\infty \leq \| F - F_0 \|_\infty. 
$$
Moreover, we have $ \| \max\{ \min\{F_\mathrm{M}, 1\}, 0 \} - F_0 \|_\infty \leq \| F - F_0 \|_\infty. $
\end{lemma}
\begin{proof}
For any $t\in\Rb$, WLOG, assume there exists $s \leq t$ s.t. $F_\mathrm{M}(t) = F(s)$. If $F(s) \geq F_0(t)$, then $ |F_\mathrm{M}(t) - F_0(t)| = F(s) - F_0(t) \leq F(s) - F_0(s) \leq \| F - F_0 \|_\infty $ since $F_0$ is monotone. If  $F(s) < F_0(t)$, we also have $ |F_\mathrm{M}(t) - F_0(t)| = F_0(t) - F(s) \leq F_0(t) - F^\pi(t) \leq \| F - F_0 \|_\infty $ since $F(s) = \sup_{t'\leq t} F(t')$. Therefore, $| F_\mathrm{M}(t) - F_0(t) | \leq \| F - F_0 \|_\infty $ for all $t\in\Rb$. The second claim is obvious.
\end{proof}

%%%%%%%%%%%%%%%%%%%%%%%%%%%%%%%%%%%%%%%%%%%%%%%%%%%%%%%%%%%%%%%%%%%%%%%%%%%%%%%%%%%%%%%%
%                                  Bibliography
%%%%%%%%%%%%%%%%%%%%%%%%%%%%%%%%%%%%%%%%%%%%%%%%%%%%%%%%%%%%%%%%%%%%%%%%%%%%%%%%%%%%%%%%
%\nocite{}                             %单独列出某一未被引用的文献
%\nocite{*}                            %列出所有未被引用的文献
%\bibliographystyle{plain}             %使用\usepackage{cite}时要用，和{plainnat}二选一
\bibliographystyle{plainnat}           %使用\usepackage{natbib}时要用
\bibliography{Bib_Wan.bib}

\begin{thebibliography}{48}
\providecommand{\natexlab}[1]{#1}
\providecommand{\url}[1]{\texttt{#1}}
\expandafter\ifx\csname urlstyle\endcsname\relax
  \providecommand{\doi}[1]{doi: #1}\else
  \providecommand{\doi}{doi: \begingroup \urlstyle{rm}\Url}\fi

\bibitem[Artzner et~al.(1999)Artzner, Delbaen, Eber, and
  Heath]{artzner1999coherent}
Philippe Artzner, Freddy Delbaen, Jean-Marc Eber, and David Heath.
\newblock Coherent measures of risk.
\newblock \emph{Mathematical Finance}, 9\penalty0 (3):\penalty0 203--228, 1999.
\newblock \doi{https://doi.org/10.1111/1467-9965.00068}.
\newblock URL
  \url{https://onlinelibrary.wiley.com/doi/abs/10.1111/1467-9965.00068}.

\bibitem[Athey and Wager(2021)]{athey2021policy}
Susan Athey and Stefan Wager.
\newblock Policy learning with observational data.
\newblock \emph{Econometrica}, 89\penalty0 (1):\penalty0 pp. 133--161, 2021.
\newblock ISSN 00129682, 14680262.
\newblock URL \url{https://www.jstor.org/stable/48628848}.

\bibitem[Bastani et~al.(2022)Bastani, Ma, Shen, and Xu]{bastani2022regret}
Osbert Bastani, Yecheng~Jason Ma, Estelle Shen, and Wanqiao Xu.
\newblock Regret bounds for risk-sensitive reinforcement learning.
\newblock In Alice~H. Oh, Alekh Agarwal, Danielle Belgrave, and Kyunghyun Cho,
  editors, \emph{Advances in Neural Information Processing Systems}, 2022.
\newblock URL \url{https://openreview.net/forum?id=yJEUDfzsTX7}.

\bibitem[Bellemare et~al.(2017)Bellemare, Dabney, and Munos]{bellemare2017a}
Marc~G. Bellemare, Will Dabney, and R{\'e}mi Munos.
\newblock A distributional perspective on reinforcement learning.
\newblock In Doina Precup and Yee~Whye Teh, editors, \emph{Proceedings of the
  34th International Conference on Machine Learning}, volume~70 of
  \emph{Proceedings of Machine Learning Research}, pages 449--458. PMLR, 06--11
  Aug 2017.
\newblock URL \url{https://proceedings.mlr.press/v70/bellemare17a.html}.

\bibitem[Bertsimas and Kallus(2019)]{bertsimas2020predictive}
Dimitris Bertsimas and Nathan Kallus.
\newblock From predictive to prescriptive analytics.
\newblock \emph{Management Science}, 66\penalty0 (3):\penalty0 1025--1044,
  2019.
\newblock URL \url{https://doi.org/10.1287/mnsc.2018.3253}.

\bibitem[Bottou et~al.(2013)Bottou, Peters, Qui{{\~n}}onero-Candela, Charles,
  Chickering, Portugaly, Ray, Simard, and Snelson]{bottou2013acounterfactual}
L{{\'e}}on Bottou, Jonas Peters, Joaquin Qui{{\~n}}onero-Candela, Denis~X.
  Charles, D.~Max Chickering, Elon Portugaly, Dipankar Ray, Patrice Simard, and
  Ed~Snelson.
\newblock Counterfactual reasoning and learning systems: The example of
  computational advertising.
\newblock \emph{Journal of Machine Learning Research}, 14\penalty0
  (101):\penalty0 3207--3260, 2013.
\newblock URL \url{http://jmlr.org/papers/v14/bottou13a.html}.

\bibitem[Buckman et~al.(2021)Buckman, Gelada, and Bellemare]{buckman2021the}
Jacob Buckman, Carles Gelada, and Marc~G Bellemare.
\newblock The importance of pessimism in fixed-dataset policy optimization.
\newblock In \emph{International Conference on Learning Representations}, 2021.
\newblock URL \url{https://openreview.net/forum?id=E3Ys6a1NTGT}.

\bibitem[Cassel et~al.(2018)Cassel, Mannor, and Zeevi]{cassel2018a}
Asaf Cassel, Shie Mannor, and Assaf Zeevi.
\newblock A general approach to multi-armed bandits under risk criteria.
\newblock In Sébastien Bubeck, Vianney Perchet, and Philippe Rigollet,
  editors, \emph{Proceedings of the 31st Conference On Learning Theory},
  volume~75 of \emph{Proceedings of Machine Learning Research}, pages
  1295--1306. PMLR, 06--09 Jul 2018.
\newblock URL \url{https://proceedings.mlr.press/v75/cassel18a.html}.

\bibitem[Chandak et~al.(2021)Chandak, Niekum, da~Silva, Learned-Miller,
  Brunskill, and Thomas]{chandak2021universal}
Yash Chandak, Scott Niekum, Bruno da~Silva, Erik Learned-Miller, Emma
  Brunskill, and Philip~S. Thomas.
\newblock Universal off-policy evaluation.
\newblock In M.~Ranzato, A.~Beygelzimer, Y.~Dauphin, P.S. Liang, and J.~Wortman
  Vaughan, editors, \emph{Advances in Neural Information Processing Systems},
  volume~34, pages 27475--27490. Curran Associates, Inc., 2021.
\newblock URL
  \url{https://proceedings.neurips.cc/paper_files/paper/2021/file/e71e5cd119bbc5797164fb0cd7fd94a4-Paper.pdf}.

\bibitem[Dabney et~al.(2018{\natexlab{a}})Dabney, Ostrovski, Silver, and
  Munos]{dabney2018implicit}
Will Dabney, Georg Ostrovski, David Silver, and Remi Munos.
\newblock Implicit quantile networks for distributional reinforcement learning.
\newblock In Jennifer Dy and Andreas Krause, editors, \emph{Proceedings of the
  35th International Conference on Machine Learning}, volume~80 of
  \emph{Proceedings of Machine Learning Research}, pages 1096--1105. PMLR,
  10--15 Jul 2018{\natexlab{a}}.
\newblock URL \url{https://proceedings.mlr.press/v80/dabney18a.html}.

\bibitem[Dabney et~al.(2018{\natexlab{b}})Dabney, Rowland, Bellemare, and
  Munos]{dabney2018distributional}
Will Dabney, Mark Rowland, Marc Bellemare, and Rémi Munos.
\newblock Distributional reinforcement learning with quantile regression.
\newblock \emph{Proceedings of the AAAI Conference on Artificial Intelligence},
  32\penalty0 (1), Apr. 2018{\natexlab{b}}.
\newblock \doi{10.1609/aaai.v32i1.11791}.
\newblock URL \url{https://ojs.aaai.org/index.php/AAAI/article/view/11791}.

\bibitem[Daniely et~al.(2011)Daniely, Sabato, Ben-David, and
  Shalev-Shwartz]{daniely2011multiclass}
Amit Daniely, Sivan Sabato, Shai Ben-David, and Shai Shalev-Shwartz.
\newblock Multiclass learnability and the erm principle.
\newblock In Sham~M. Kakade and Ulrike von Luxburg, editors, \emph{Proceedings
  of the 24th Annual Conference on Learning Theory}, volume~19 of
  \emph{Proceedings of Machine Learning Research}, pages 207--232, Budapest,
  Hungary, 09--11 Jun 2011. PMLR.
\newblock URL \url{https://proceedings.mlr.press/v19/daniely11a.html}.

\bibitem[Dud\'{\i}k et~al.(2011)Dud\'{\i}k, Langford, and Li]{dudik2011doubly}
Miroslav Dud\'{\i}k, John Langford, and Lihong Li.
\newblock Doubly robust policy evaluation and learning.
\newblock In \emph{Proceedings of the 28th International Conference on
  International Conference on Machine Learning}, ICML'11, page 1097–1104,
  Madison, WI, USA, 2011. Omnipress.
\newblock ISBN 9781450306195.

\bibitem[Fei and Xu(2022)]{fei2022cascaded}
Yingjie Fei and Ruitu Xu.
\newblock Cascaded gaps: Towards logarithmic regret for risk-sensitive
  reinforcement learning.
\newblock In Kamalika Chaudhuri, Stefanie Jegelka, Le~Song, Csaba Szepesvari,
  Gang Niu, and Sivan Sabato, editors, \emph{Proceedings of the 39th
  International Conference on Machine Learning}, volume 162 of
  \emph{Proceedings of Machine Learning Research}, pages 6392--6417. PMLR,
  17--23 Jul 2022.
\newblock URL \url{https://proceedings.mlr.press/v162/fei22b.html}.

\bibitem[Fei et~al.(2020)Fei, Yang, Chen, Wang, and Xie]{fei2020risk}
Yingjie Fei, Zhuoran Yang, Yudong Chen, Zhaoran Wang, and Qiaomin Xie.
\newblock Risk-sensitive reinforcement learning: near-optimal risk-sample
  tradeoff in regret.
\newblock In \emph{Proceedings of the 34th International Conference on Neural
  Information Processing Systems}, NIPS '20, Red Hook, NY, USA, 2020. Curran
  Associates Inc.
\newblock ISBN 9781713829546.

\bibitem[Howard and Matheson(1972)]{howard1972risk-sensitive}
Ronald~A. Howard and James~E. Matheson.
\newblock Risk-sensitive markov decision processes.
\newblock \emph{Management Science}, 18\penalty0 (7):\penalty0 356--369, 1972.
\newblock ISSN 00251909, 15265501.
\newblock URL \url{http://www.jstor.org/stable/2629352}.

\bibitem[Huang et~al.(2021)Huang, Leqi, Lipton, and
  Azizzadenesheli]{huang2021off}
Audrey Huang, Liu Leqi, Zachary Lipton, and Kamyar Azizzadenesheli.
\newblock Off-policy risk assessment in contextual bandits.
\newblock In M.~Ranzato, A.~Beygelzimer, Y.~Dauphin, P.S. Liang, and J.~Wortman
  Vaughan, editors, \emph{Advances in Neural Information Processing Systems},
  volume~34, pages 23714--23726. Curran Associates, Inc., 2021.
\newblock URL
  \url{https://proceedings.neurips.cc/paper_files/paper/2021/file/c7502c55f8db540625b59d9a42638520-Paper.pdf}.

\bibitem[Huang et~al.(2022)Huang, Leqi, Lipton, and
  Azizzadenesheli]{huang2022offpolicy}
Audrey Huang, Liu Leqi, Zachary Lipton, and Kamyar Azizzadenesheli.
\newblock Off-policy risk assessment for markov decision processes.
\newblock In Gustau Camps-Valls, Francisco J.~R. Ruiz, and Isabel Valera,
  editors, \emph{Proceedings of The 25th International Conference on Artificial
  Intelligence and Statistics}, volume 151 of \emph{Proceedings of Machine
  Learning Research}, pages 5022--5050. PMLR, 28--30 Mar 2022.
\newblock URL \url{https://proceedings.mlr.press/v151/huang22b.html}.

\bibitem[Jin(2023)]{jin2023natarajan}
Ying Jin.
\newblock Upper bounds on the natarajan dimensions of some function classes.
\newblock In \emph{2023 IEEE International Symposium on Information Theory
  (ISIT)}, pages 1020--1025, 2023.
\newblock \doi{10.1109/ISIT54713.2023.10206618}.

\bibitem[Jin et~al.(2021)Jin, Yang, and Wang]{jin2021is}
Ying Jin, Zhuoran Yang, and Zhaoran Wang.
\newblock Is pessimism provably efficient for offline rl?
\newblock In Marina Meila and Tong Zhang, editors, \emph{Proceedings of the
  38th International Conference on Machine Learning}, volume 139 of
  \emph{Proceedings of Machine Learning Research}, pages 5084--5096. PMLR,
  18--24 Jul 2021.
\newblock URL \url{https://proceedings.mlr.press/v139/jin21e.html}.

\bibitem[Jin et~al.(2025)Jin, Ren, Yang, and Wang]{jin2025policy}
Ying Jin, Zhimei Ren, Zhuoran Yang, and Zhaoran Wang.
\newblock {Policy learning “without” overlap: Pessimism and generalized
  empirical Bernstein’s inequality}.
\newblock \emph{The Annals of Statistics}, 53\penalty0 (4):\penalty0 1483 --
  1512, 2025.
\newblock \doi{10.1214/25-AOS2511}.
\newblock URL \url{https://doi.org/10.1214/25-AOS2511}.

\bibitem[Kallus(2018)]{kallus2018balanced}
Nathan Kallus.
\newblock Balanced policy evaluation and learning.
\newblock In S.~Bengio, H.~Wallach, H.~Larochelle, K.~Grauman, N.~Cesa-Bianchi,
  and R.~Garnett, editors, \emph{Advances in Neural Information Processing
  Systems}, volume~31. Curran Associates, Inc., 2018.
\newblock URL
  \url{https://proceedings.neurips.cc/paper_files/paper/2018/file/6616758da438b02b8d360ad83a5b3d77-Paper.pdf}.

\bibitem[Kennedy(2019)]{kennedy2019nonparametric}
Edward~H. Kennedy.
\newblock Nonparametric causal effects based on incremental propensity score
  interventions.
\newblock \emph{Journal of the American Statistical Association}, 114\penalty0
  (526):\penalty0 645--656, 2019.
\newblock \doi{10.1080/01621459.2017.1422737}.
\newblock URL \url{https://doi.org/10.1080/01621459.2017.1422737}.

\bibitem[Keramati et~al.(2020)Keramati, Dann, Tamkin, and
  Brunskill]{keramati2020being}
Ramtin Keramati, Christoph Dann, Alex Tamkin, and Emma Brunskill.
\newblock Being optimistic to be conservative: Quickly learning a cvar policy.
\newblock \emph{Proceedings of the AAAI Conference on Artificial Intelligence},
  34\penalty0 (04):\penalty0 4436--4443, Apr. 2020.
\newblock \doi{10.1609/aaai.v34i04.5870}.
\newblock URL \url{https://ojs.aaai.org/index.php/AAAI/article/view/5870}.

\bibitem[Kitagawa and Tetenov(2018)]{kitagawa2018who}
Toru Kitagawa and Aleksey Tetenov.
\newblock Who should be treated? empirical welfare maximization methods for
  treatment choice.
\newblock \emph{Econometrica}, 86\penalty0 (2):\penalty0 591--616, 2018.
\newblock ISSN 00129682, 14680262.
\newblock URL \url{http://www.jstor.org/stable/44955978}.

\bibitem[Kusuoka(2001)]{kusuoka2001on}
Shigeo Kusuoka.
\newblock \emph{On law invariant coherent risk measures}, pages 83--95.
\newblock Springer Japan, Tokyo, 2001.
\newblock ISBN 978-4-431-67891-5.
\newblock \doi{10.1007/978-4-431-67891-5_4}.
\newblock URL \url{https://doi.org/10.1007/978-4-431-67891-5_4}.

\bibitem[L.A. and Bhat(2022)]{prashanth2022a}
Prashanth L.A. and Sanjay~P. Bhat.
\newblock A wasserstein distance approach for concentration of empirical risk
  estimates.
\newblock \emph{Journal of Machine Learning Research}, 23\penalty0
  (238):\penalty0 1--61, 2022.
\newblock URL \url{http://jmlr.org/papers/v23/20-965.html}.

\bibitem[L.A. et~al.(2016)L.A., Jie, Fu, Marcus, and
  Szepesvari]{prashanth2016cumulative}
Prashanth L.A., Cheng Jie, Michael Fu, Steve Marcus, and Csaba Szepesvari.
\newblock Cumulative prospect theory meets reinforcement learning: Prediction
  and control.
\newblock In Maria~Florina Balcan and Kilian~Q. Weinberger, editors,
  \emph{Proceedings of The 33rd International Conference on Machine Learning},
  volume~48 of \emph{Proceedings of Machine Learning Research}, pages
  1406--1415, New York, New York, USA, 20--22 Jun 2016. PMLR.
\newblock URL \url{https://proceedings.mlr.press/v48/la16.html}.

\bibitem[Lattimore and Szepesvári(2020)]{lattimore2020bandit}
Tor Lattimore and Csaba Szepesvári.
\newblock \emph{Bandit Algorithms}.
\newblock Cambridge University Press, 2020.

\bibitem[Li et~al.(2010)Li, Chu, Langford, and Schapire]{li2010a}
Lihong Li, Wei Chu, John Langford, and Robert~E. Schapire.
\newblock A contextual-bandit approach to personalized news article
  recommendation.
\newblock In \emph{Proceedings of the 19th International Conference on World
  Wide Web}, WWW '10, page 661–670, New York, NY, USA, 2010. Association for
  Computing Machinery.
\newblock ISBN 9781605587998.
\newblock \doi{10.1145/1772690.1772758}.
\newblock URL \url{https://doi.org/10.1145/1772690.1772758}.

\bibitem[Liang and Luo(2024)]{liang2024bridging}
Hao Liang and Zhi-Quan Luo.
\newblock Bridging distributional and risk-sensitive reinforcement learning
  with provable regret bounds.
\newblock \emph{Journal of Machine Learning Research}, 25\penalty0
  (221):\penalty0 1--56, 2024.
\newblock URL \url{http://jmlr.org/papers/v25/22-1253.html}.

\bibitem[Ma et~al.(2021)Ma, Jayaraman, and Bastani]{ma2021conservative}
Yecheng Ma, Dinesh Jayaraman, and Osbert Bastani.
\newblock Conservative offline distributional reinforcement learning.
\newblock In M.~Ranzato, A.~Beygelzimer, Y.~Dauphin, P.S. Liang, and J.~Wortman
  Vaughan, editors, \emph{Advances in Neural Information Processing Systems},
  volume~34, pages 19235--19247. Curran Associates, Inc., 2021.
\newblock URL
  \url{https://proceedings.neurips.cc/paper_files/paper/2021/file/a05d886123a54de3ca4b0985b718fb9b-Paper.pdf}.

\bibitem[Mannor and Tsitsiklis(2011)]{mannor2011mean}
Shie Mannor and John~N Tsitsiklis.
\newblock Mean-variance optimization in markov decision processes.
\newblock In \emph{Proceedings of the 28th International Conference on
  International Conference on Machine Learning}, pages 177--184, 2011.

\bibitem[Murphy(2003)]{murphy2003optimal}
S.~A. Murphy.
\newblock Optimal dynamic treatment regimes.
\newblock \emph{Journal of the Royal Statistical Society. Series B (Statistical
  Methodology)}, 65\penalty0 (2):\penalty0 331--366, 2003.
\newblock ISSN 13697412, 14679868.
\newblock URL \url{http://www.jstor.org/stable/3647509}.

\bibitem[Natarajan(1989)]{natarajan1989learning}
Balas~K Natarajan.
\newblock On learning sets and functions.
\newblock \emph{Machine Learning}, 4\penalty0 (1):\penalty0 67--97, 1989.

\bibitem[Precup et~al.(2000)Precup, Sutton, and Singh]{precup2000eligibility}
Doina Precup, Richard~S Sutton, and Satinder Singh.
\newblock Eligibility traces for off-policy policy evaluation.
\newblock In \emph{ICML}, volume 2000, pages 759--766. Citeseer, 2000.

\bibitem[Rigter et~al.(2023)Rigter, Lacerda, and Hawes]{rigter2023one}
Marc Rigter, Bruno Lacerda, and Nick Hawes.
\newblock One risk to rule them all: A risk-sensitive perspective on
  model-based offline reinforcement learning.
\newblock In A.~Oh, T.~Naumann, A.~Globerson, K.~Saenko, M.~Hardt, and
  S.~Levine, editors, \emph{Advances in Neural Information Processing Systems},
  volume~36, pages 77520--77545. Curran Associates, Inc., 2023.
\newblock URL
  \url{https://proceedings.neurips.cc/paper_files/paper/2023/file/f49287371916715b9209fa41a275851e-Paper-Conference.pdf}.

\bibitem[Rockafellar et~al.(2000)Rockafellar, Uryasev,
  et~al.]{rockafellar2000optimization}
R~Tyrrell Rockafellar, Stanislav Uryasev, et~al.
\newblock Optimization of conditional value-at-risk.
\newblock \emph{Journal of risk}, 2:\penalty0 21--42, 2000.

\bibitem[Sani et~al.(2012)Sani, Lazaric, and Munos]{sani2012risk}
Amir Sani, Alessandro Lazaric, and R\'{e}mi Munos.
\newblock Risk-aversion in multi-armed bandits.
\newblock In F.~Pereira, C.J. Burges, L.~Bottou, and K.Q. Weinberger, editors,
  \emph{Advances in Neural Information Processing Systems}, volume~25. Curran
  Associates, Inc., 2012.
\newblock URL
  \url{https://proceedings.neurips.cc/paper_files/paper/2012/file/83f2550373f2f19492aa30fbd5b57512-Paper.pdf}.

\bibitem[Swaminathan and Joachims(2015)]{swaminathan2015batch}
Adith Swaminathan and Thorsten Joachims.
\newblock Batch learning from logged bandit feedback through counterfactual
  risk minimization.
\newblock \emph{Journal of Machine Learning Research}, 16\penalty0
  (52):\penalty0 1731--1755, 2015.
\newblock URL \url{http://jmlr.org/papers/v16/swaminathan15a.html}.

\bibitem[Tamar et~al.(2016)Tamar, Castro, and Mannor]{tamar2016learning}
Aviv Tamar, Dotan~Di Castro, and Shie Mannor.
\newblock Learning the variance of the reward-to-go.
\newblock \emph{Journal of Machine Learning Research}, 17\penalty0
  (13):\penalty0 1--36, 2016.
\newblock URL \url{http://jmlr.org/papers/v17/14-335.html}.

\bibitem[Urp{\'\i} et~al.(2021)Urp{\'\i}, Curi, and
  Krause]{nuria2021riskaverse}
N{\'u}ria~Armengol Urp{\'\i}, Sebastian Curi, and Andreas Krause.
\newblock Risk-averse offline reinforcement learning.
\newblock In \emph{International Conference on Learning Representations}, 2021.
\newblock URL \url{https://openreview.net/forum?id=TBIzh9b5eaz}.

\bibitem[Villani(2009)]{villani2009the}
C{\'e}dric Villani.
\newblock \emph{The Wasserstein distances}, pages 93--111.
\newblock Springer Berlin Heidelberg, Berlin, Heidelberg, 2009.
\newblock ISBN 978-3-540-71050-9.
\newblock \doi{10.1007/978-3-540-71050-9_6}.
\newblock URL \url{https://doi.org/10.1007/978-3-540-71050-9_6}.

\bibitem[Wang et~al.(2023)Wang, Kallus, and Sun]{wang2023near}
Kaiwen Wang, Nathan Kallus, and Wen Sun.
\newblock Near-minimax-optimal risk-sensitive reinforcement learning with
  {CV}a{R}.
\newblock In Andreas Krause, Emma Brunskill, Kyunghyun Cho, Barbara Engelhardt,
  Sivan Sabato, and Jonathan Scarlett, editors, \emph{Proceedings of the 40th
  International Conference on Machine Learning}, volume 202 of
  \emph{Proceedings of Machine Learning Research}, pages 35864--35907. PMLR,
  23--29 Jul 2023.
\newblock URL \url{https://proceedings.mlr.press/v202/wang23m.html}.

\bibitem[Zhan et~al.(2024)Zhan, Ren, Athey, and Zhou]{zhan2024policy}
Ruohan Zhan, Zhimei Ren, Susan Athey, and Zhengyuan Zhou.
\newblock Policy learning with adaptively collected data.
\newblock \emph{Management Science}, 70\penalty0 (8):\penalty0 5270--5297,
  2024.
\newblock URL \url{https://doi.org/10.1287/mnsc.2023.4921}.

\bibitem[Zhang et~al.(2024)Zhang, Lyu, Qiu, Kolar, and
  Zhang]{zhang2024pessimism}
Dake Zhang, Boxiang Lyu, Shuang Qiu, Mladen Kolar, and Tong Zhang.
\newblock Pessimism meets risk: Risk-sensitive offline reinforcement learning.
\newblock In Ruslan Salakhutdinov, Zico Kolter, Katherine Heller, Adrian
  Weller, Nuria Oliver, Jonathan Scarlett, and Felix Berkenkamp, editors,
  \emph{Proceedings of the 41st International Conference on Machine Learning},
  volume 235 of \emph{Proceedings of Machine Learning Research}, pages
  59459--59489. PMLR, 21--27 Jul 2024.
\newblock URL \url{https://proceedings.mlr.press/v235/zhang24aq.html}.

\bibitem[Zhao et~al.(2024)Zhao, Chambaz, Josse, and Yang]{zhao2024positivity}
Pan Zhao, Antoine Chambaz, Julie Josse, and Shu Yang.
\newblock Positivity-free policy learning with observational data.
\newblock In Sanjoy Dasgupta, Stephan Mandt, and Yingzhen Li, editors,
  \emph{Proceedings of The 27th International Conference on Artificial
  Intelligence and Statistics}, volume 238 of \emph{Proceedings of Machine
  Learning Research}, pages 1918--1926. PMLR, 02--04 May 2024.
\newblock URL \url{https://proceedings.mlr.press/v238/zhao24a.html}.

\bibitem[Zhou et~al.(2023)Zhou, Athey, and Wager]{zhou2023offline}
Zhengyuan Zhou, Susan Athey, and Stefan Wager.
\newblock Offline multi-action policy learning: Generalization and
  optimization.
\newblock \emph{Operations Research}, 71\penalty0 (1):\penalty0 148--183, 2023.
\newblock \doi{10.1287/opre.2022.2271}.
\newblock URL \url{https://doi.org/10.1287/opre.2022.2271}.

\end{thebibliography}

%%%%%%%%%%%%%%%%%%%%%%%%%%%%%%%%%%%%%%%%%%%%%%%%%%%%%%%%%%%%%%%%%%%%%%%%%%%%%%%
%%%%%%%%%%%%%%%%%%%%%%%%%%%%%%%%%%%%%%%%%%%%%%%%%%%%%%%%%%%%%%%%%%%%%%%%%%%%%%%

\end{document}